\documentclass[11pt]{article}
\usepackage[a4paper,margin=1in]{geometry}
\usepackage[T1]{fontenc}
\usepackage[utf8]{inputenc}
\usepackage{lmodern,amsmath,amssymb,amsthm,mathtools}
\usepackage{booktabs,tabularx,longtable,array,enumitem,microtype}
\usepackage[dvipsnames]{xcolor}
\usepackage{tikz}
\usetikzlibrary{arrows.meta,positioning}
\usepackage{pgfplots}
\pgfplotsset{compat=1.18}
\definecolor{fiNative}{HTML}{315A83}
\definecolor{fiFrontier}{HTML}{177F7B}
\definecolor{fiAccent}{HTML}{B76A20}
\definecolor{fiGray}{HTML}{68717A}
\usepackage[hidelinks]{hyperref}
\hypersetup{
pdftitle={Verification and Self-Improvement in Agentic AI: Foundations and Limits},
pdfauthor={Chien-Ping Lu},
pdfsubject={Foundations and limits of bounded verification and recursive self-improvement in agentic AI},
pdfkeywords={agentic AI, bounded verification, recursive self-improvement, closure frontiers, polynomial hierarchy}}
\newcommand{\Pclass}{\mathsf{P}}
\newcommand{\NP}{\mathsf{NP}}
\newcommand{\BPP}{\mathsf{BPP}}
\newcommand{\PH}{\mathsf{PH}}
\newcommand{\MA}{\mathsf{MA}}

\newcommand{\SP}[1]{\Sigma_{#1}^{\mathsf P}}
\newcommand{\PP}[1]{\Pi_{#1}^{\mathsf P}}
\newcommand{\SR}[1]{\Sigma_{#1}^{\mathsf{RV}}}
\newcommand{\PR}[1]{\Pi_{#1}^{\mathsf{RV}}}
\newcommand{\Nat}{\mathrm{Nat}}
\newcommand{\Exp}{\mathrm{Exp}}
\newcommand{\Core}{\mathrm{Core}}
\newcommand{\Front}{\mathrm F}
\newcommand{\Lang}{\mathrm L}
\newcommand{\eps}{\varepsilon}
\newcommand{\bits}{\{0,1\}}
\newcommand{\Q}{\mathcal Q}
\newcommand{\Acc}{\mathcal A}
\newtheorem{theorem}{Theorem}[section]
\newtheorem{proposition}[theorem]{Proposition}

\theoremstyle{definition}
\newtheorem{definition}[theorem]{Definition}

\theoremstyle{remark}

\setlist{topsep=4pt,itemsep=2pt}
\title{Verification and Self-Improvement in Agentic AI:\\Foundations and Limits}
\author{Chien-Ping Lu}
\date{}
\begin{document}
\maketitle
\begin{abstract}
\noindent A self-improving agent can find accepted solutions more often without expanding the tasks its verification interface can accept. We extend task-relative verification with explicit support and transcript interfaces to distinguish search gains from changes in accepted-task coverage. A stage's native reach contains tasks accepted with default support; its closure frontier contains tasks accepted with any support the declared interface permits. A semantic containment criterion identifies the authority or exposure condition that every verified escape must violate, while guard-transfer invariance prevents a change of representation from being counted as a capability expansion. An executable Boolean-program repair agent makes the distinction concrete: over 36,864 paired runs, counterexample feedback reduces verifier calls without changing coverage, whereas an enforced permission change admits a new target with a proof excluding every old-interface program. Exact and uniformly gapped randomized verification share this account; independent majority amplification improves reliability while preserving the accepted language. Separate adaptive-selection contracts state what is additionally required to interpret a proposed successor's score as evidence of deployment improvement. Classical hierarchy results and finite audits provide supporting boundaries. The framework specifies the evidence needed to establish improved search, wider accepted capability, and reliable successor selection.
\end{abstract}

\section{Introduction}
A coding agent proposes patches, receives feedback, and revises its search procedure. After an update, it solves more benchmark instances. Does the update make old solutions easier to find, make previously optional support routine, or admit solutions that the old interface could never accept? These mechanisms can produce similar score changes while requiring different explanations and evidence.

We study the accepted-task question. A deployed stage specifies a correctness relation, permitted support, admissible interaction records, and a bounded terminal checker. Its \emph{native reach} records the tasks with an accepted solution under default support. Its \emph{closure frontier} records tasks with an accepted solution under any support that the current interface permits. The frontier is the entire accepted-input set, not a topological boundary or an empirical score. Search effectiveness remains a separate probability or cost measure.

The central comparison asks what a proved expansion must change. If a successor's accepted transcripts remain within the predecessor's semantic authority and its default records remain within the predecessor's permitted exposure, the successor's native reach stays inside the old frontier. An escape witness must violate at least one of these conditions. The criterion is elementary set semantics; its purpose is to make the evidence obligation explicit. Guard-transfer invariance ensures that moving a check between the harness and terminal verifier does not count as expansion when accepted behavior is unchanged.

The program-repair construction supplies a concrete use of this distinction. True counterexamples reduce verifier calls while preserving the first correct candidate and accepted coverage. Changing an enforced permission admits a particular new target, and a growing-family exclusion argument rules out every old-interface program for that target. The 36,864 paired runs measure the search benefit; the general exclusion proof establishes the interface separation. Neither substitutes for the other.

We also separate two claims often attached to self-improvement. Randomized verification requires a justified connection between acceptance probability and correctness: independent amplification of the same tuple preserves its accepted language. Selecting a successor requires its own error and deployment-transfer contract; a better score alone proves no language containment or escape. The detailed hierarchy analysis uses classical complexity tools and does not establish new class separations. Contemporary agent and improver updates motivate these comparisons \cite{Zhang2026DGM,Zelikman2024STOP,Zhang2026Hyperagents,Duan2026LastAI}, but the structural results concern declared families and do not bound the rate of practical improvement.

Section~\ref{sec:stage} defines the contract. Section~\ref{sec:structure} gives the containment and representation criteria, and Section~\ref{sec:repair} demonstrates search improvement and a proved expansion. Section~\ref{sec:random} explains reliability and successor evaluation. Full proofs, classical placement and depth results, constructions, and experimental protocols are in the appendices.

\begin{figure}[tbp]
\centering
\begin{tikzpicture}[x=1cm,y=1cm,font=\small,
 old/.style={draw=fiGray,densely dashed,line width=.8pt},
 frontier/.style={draw=fiFrontier,line width=1.1pt},
 native/.style={draw=fiNative,fill=fiNative!12,line width=.8pt}]
\foreach \x/\heading in {0/{(a) Better search},5.15/{(b) Internalization},10.3/{(c) Frontier expansion}} {
  \node[font=\small\bfseries,anchor=west] at (\x,3.2) {\heading};
}
\begin{scope}[shift={(2.1,1.65)}]
  \draw[frontier,fill=fiFrontier!4] (0,0) ellipse (1.95 and 1.1);
  \draw[native] (0,-.05) ellipse (.78 and .5);
  \node at (0,-.05) {$L'=L$};
  \node[text=fiFrontier] at (0,.74) {$F'=F$};
\end{scope}
\node[align=center,text width=4.3cm] at (2.1,-.05)
  {Higher success probability;\\same accepted-input sets.};
\begin{scope}[shift={(7.25,1.65)}]
  \draw[frontier,fill=fiFrontier!4] (0,0) ellipse (1.95 and 1.1);
  \draw[native] (0,-.12) ellipse (1.35 and .68);
  \draw[old] (-.24,-.18) ellipse (.58 and .34);
  \node at (-.24,-.18) {$L$};
  \node[text=fiNative] at (.8,-.1) {$L'$};
  \node[text=fiFrontier] at (0,.74) {$F'=F$};
\end{scope}
\node[align=center,text width=4.3cm] at (7.25,-.05)
  {$L\subsetneq L'\subseteq F$;\\more default access.};
\begin{scope}[shift={(12.4,1.65)}]
  \draw[frontier,fill=fiFrontier!4] (0,0) ellipse (1.95 and 1.1);
  \draw[native] (.36,-.13) ellipse (1.34 and .66);
  \draw[old] (-.4,-.13) ellipse (1.15 and .7);
  \node[text=fiGray] at (-1,.27) {$F$};
  \node[text=fiFrontier] at (0,.82) {$F'$};
  \node[text=fiNative] at (.5,-.16) {$L'$};
  \fill[fiAccent] (1.39,-.13) circle (2pt);
  \node[text=fiAccent,below] at (1.39,-.2) {$x_\star$};
\end{scope}
\node[align=center,text width=4.3cm] at (12.4,-.05)
  {$x_\star\in L'\setminus F$;\\beyond the old frontier.};
\draw[fiNative,fill=fiNative!12] (2.15,-.92) rectangle (2.48,-.75);
\node[anchor=west,font=\footnotesize] at (2.58,-.84) {native reach};
\draw[frontier] (6.05,-.84)--(6.5,-.84);
\node[anchor=west,font=\footnotesize] at (6.6,-.84) {closure frontier};
\draw[old] (10.75,-.84)--(11.2,-.84);
\node[anchor=west,font=\footnotesize] at (11.3,-.84) {previous boundary};
\end{tikzpicture}
\caption{Three comparisons under a common task relation. Each shape represents
\emph{task inputs}; area is schematic and does not measure probability or
cardinality. Write $L=\Lang_R(S)$, $F=\Front_R(S)$ and use primes for the
successor. Better search can preserve both languages. Internalization enlarges
native reach within unchanged optional support. The illustrated frontier
expansion admits an input natively that has no accepted source-frontier
witness. This nested drawing is one possible transition, not an assumption
that every stage change is monotone.}
\label{fig:frontier-changes}
\end{figure}
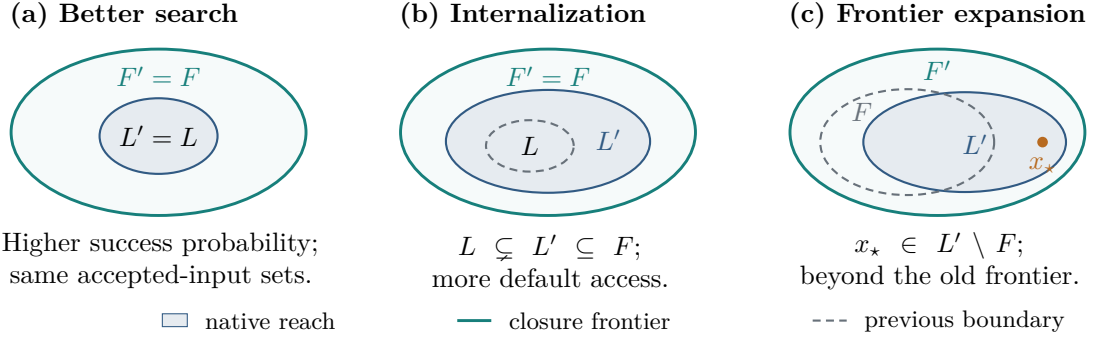

\section{A contract for comparing accepted capability}\label{sec:stage}
A stage $S$ fixes a task relation $R(x,z)$, default support $\sigma_0$, an exposure predicate $\Exp_S(x,\sigma,t)$, and a terminal core $\Core_S(x,z,t,\vec y;r)$. Support, transcripts, outputs, and verification moves have declared polynomial length bounds. Exposure and the core are jointly uniform polynomial-time procedures; malformed or out-of-bound records are rejected. Default support is uniformly polynomial-time computable. Support supplies bounded auxiliary data. All move domains include the empty string.

Let $\Q_k\vec y$ denote $k\ge1$ alternating bounded verification blocks beginning existentially. All moves precede the uniform hidden random tape $r$. On every well-formed bounded tuple, the core acceptance probability lies in $[0,1/3]\cup[2/3,1]$. Define
\begin{align*}
D_S(x,z,t,\vec y)&=\mathbf 1\{\Pr_r[\Core_S(x,z,t,\vec y;r)=1]\ge2/3\},\\
\Acc_S(x,z,t)&\ \Longleftrightarrow\ \Q_k\vec y\ D_S(x,z,t,\vec y)=1.
\end{align*}
Task-relative soundness requires
\[
\Exp_S(x,\sigma,t)\wedge\Acc_S(x,z,t)\quad\Longrightarrow\quad R(x,z).
\]
The gap is an assumption on every bounded tuple, and soundness connects the complete quantified acceptance game to correctness. Neither follows from the other. Setting the coin length to zero recovers exact verification. Definition~\ref{def:stage} in Appendix~\ref{app:stage-contract} gives the full bounds, encoding checks, and probability contract.

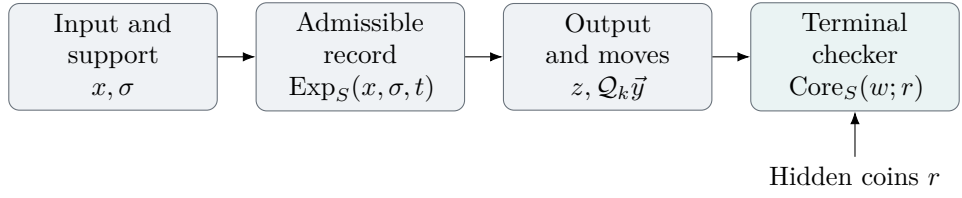
\begin{figure}[htbp]
\centering
\begin{tikzpicture}[>=Latex,font=\small,
b/.style={draw=fiGray,rounded corners,align=center,text width=2.5cm,minimum height=1cm},
node distance=5mm]
\node[b,fill=fiNative!8] (in) {Input and support\\$x,\sigma$};
\node[b,fill=fiNative!8,right=of in] (ex) {Admissible record\\$\Exp_S(x,\sigma,t)$};
\node[b,fill=fiNative!8,right=of ex] (mv) {Output and moves\\$z,\Q_k\vec y$};
\node[b,fill=fiFrontier!9,right=of mv] (co) {Terminal checker\\$\Core_S(w;r)$};
\draw[->] (in)--(ex); \draw[->] (ex)--(mv); \draw[->] (mv)--(co);
\node[below=6mm of co,align=center] (r) {Hidden coins $r$};
\draw[->] (r)--(co);
\end{tikzpicture}
\caption{The declared verification interface. Exposure and length bounds fix
the admissible records; alternating moves specify the acceptance game. Only
the terminal checker receives hidden random bits. Setting the coin length to
zero recovers deterministic verification. The diagram specifies semantics,
not the cost of solving the quantified game.}\label{fig:flow}
\end{figure}

\begin{definition}[Native reach and closure frontier]\label{def:frontier}
Let
\[
\Nat_S(x)=\{t:\Exp_S(x,\sigma_0,t)\},\qquad
\Exp_S^*(x)=\{t:\exists\sigma\ \Exp_S(x,\sigma,t)\}.
\]
All objects are bounded as in Definition~\ref{def:stage}. Define
\begin{align}
\Lang_R(S)&=\{x:\exists t\in\Nat_S(x)\ \exists z\ \Acc_S(x,z,t)\},\label{eq:native}\\
\Front_R(S)&=\{x:\exists t\in\Exp_S^*(x)\ \exists z\ \Acc_S(x,z,t)\}.\label{eq:frontier}
\end{align}
These symbols use bounded-error acceptance throughout; the superscript
$\mathrm{ex}$ will denote an exact specialization when comparison is needed.
\end{definition}

For a concrete instance, let $x=(1^d,a)$ specify a Boolean linear function,
and let $t=z$ list its referenced variables. Optional support selects a quota
$q\le B(d)$, defaulting to $b_0(d)$. Exposure requires $|t|\le q$; an exact
coefficient check establishes correctness. Native reach then requires
$\|a\|_0\le b_0(d)$, and the frontier requires $\|a\|_0\le B(d)$.
Section~\ref{sec:xor-case} specifies and proves this instance.

These are languages of \emph{task inputs}, not sets of accepted patches or
measures of success probability. Existence of an accepted witness is weaker
than efficient production of that witness by a given learned policy. A support
source must be bounded and its admissibility polynomially checkable under the
declared interface. A fact supplied by an unmodeled oracle cannot be called
conservative merely because it fits in a short string.

\section{What a verified expansion must change}\label{sec:structure}
Define the transcript-indexed accepted-input set
$K_S(x,t)\iff\exists z\,\Acc_S(x,z,t)$. A comparison uses a fixed common
transcript encoding, including rejection outside the declared domains.
When encodings differ, an explicit translation is required.

\begin{definition}[Semantic authority]\label{def:authority}
Write $T\preceq^{\rm auth}S$ if $K_T(x,t)\Rightarrow K_S(x,t)$ for all
$x,t$. This compares quantified acceptance, possibly across different depths.
\end{definition}
Authority here is acceptance of task inputs per transcript; it is not raw
model accuracy, nor necessarily containment of every accepted output. Both
stages must be sound for the same task. Opposite inclusions define equivalence.

A comparison should also survive changes of encoding. Polynomial transcript
translations and a representation using just default and optional support
preserve the two languages under the conditions developed in
Appendix~\ref{app:normalization}.

\begin{proposition}[Guard transfer]\label{prop:guard}
For a deterministic polynomial-time guard $g(x,t)$, the stages
\[
(\Exp_S\wedge g,\Core_S)\quad\hbox{and}\quad
(\Exp_S,g\wedge\Core_S)
\]
have the same native reach and frontier, with all other data fixed.
\end{proposition}
A guard can reject the same records at either interface without changing
accepted behavior. Appendix~\ref{app:guard-proof} proves the result;
Appendix~\ref{app:local-invariant} characterizes the associated local invariant.

This invariance explains an audit limitation. With one input and two
transcripts $a,b$, let a source expose and accept only $a$. A target that
exposes only $a$ and whose core accepts both transcripts has a raw authority
change. Moving the guard for $a$ into the target core produces an equivalent
target that exposes both transcripts but accepts only $a$; now the raw change
is exposure. Neither transition escapes the source frontier. Component labels
are representation-dependent even when accepted behavior is unchanged.

\begin{theorem}[Semantic containment and escape diagnostic]\label{thm:diagnostic}
If $T\preceq^{\rm auth}S$ and
$\Nat_T(x)\subseteq\Exp_S^*(x)$ for every $x$, then
$\Lang_R(T)\subseteq\Front_R(S)$, regardless of their depths.
More precisely, if $x\in\Lang_R(T)\setminus\Front_R(S)$, every accepted
target-native transcript $t$ for that input satisfies
\[
t\notin\Exp_S^*(x)\quad\hbox{or}\quad \neg K_S(x,t).
\]
\end{theorem}
The proof is in Appendix~\ref{app:diagnostic-proof}.

This is an elementary semantic diagnostic. It says exactly what an escape
witness must violate; it does not recover those premises from benchmark scores.
In particular, authority-preserving internalization of permitted support can
increase native reach only within the old frontier. The familiar three-field
report---depth, authority, exposure---records a useful architecture description,
but is not a theorem that three causes are independent. Because authority
already compares fully quantified acceptance, increased depth cannot produce
escape while both semantic containment premises hold. Depth obstructions need
the additional class argument of Appendix~\ref{sec:limits}.

Either core change or exposure change can realize frontier escape at fixed
depth, for exact and randomized stages. Appendix~\ref{app:realizations} gives
separate constructions that keep the other component fixed.
Strict escape is
stronger than a changed component identifier or a failed global containment
test: the latter can occur on an input that another source transcript already
solves. Appendix~\ref{app:chain} extends the construction to finite chains.

\section{Program repair: faster search and a proved expansion}\label{sec:repair}
This is a controlled symbolic computational experiment. Its proposal policies
are explicit search procedures, and no neural model is trained or invoked.
The restricted program grammar and complete truth-table specifications make
exact coverage comparisons possible. The runs measure search behavior; a
separate invariant proof excludes every old-interface solution for a new
target. Neither the finite measurements nor that task-specific proof establish
a complexity-class separation or learned recursive self-improvement.

\subsection{Programs, permissions, and exact evidence}
For brevity, $L$ and $F$ in the finite counts below denote $\Lang_R$ and
$\Front_R$ for this task family.
Consider Boolean decision programs on $d=2^h\ge4$ inputs, represented by a
fixed complete binary tree. Leaves read the inputs once, in order. Each of
the $d-1$ internal nodes uses AND, OR, or XOR. The initial program uses AND
at every node. A support mask grants permission to change selected operators;
the harness rejects changes at other nodes and rejects all other syntax.
Constants, negations, leaf reordering, and arbitrary code execution are outside
the declared language. Generated programs are executable Python functions.

A task input is $x=(d,T)$, where $T$ is an explicit $2^d$-bit truth table.
An output $z$ encodes the operators; $R(x,z)$ means agreement with $T$ on every
Boolean input. The terminal checker evaluates the entire table and binds the
output to the final program record. The transcript is this final record;
search traces are logged separately and do not restrict exposure. Hashes
identify code and target files in the artifact; acceptance rechecks content
rather than trusting a claimed hash.

The old maximum mask $M_{\rm leaf}$ permits edits only at the $d/2$ parents
of leaves. Default support permits one such node. The expanded maximum mask
$M_{\rm all}$ permits every internal node. Any subset of a maximum mask is
admissible support. The resulting frontier is exactly the set of tables
computed by permitted programs. Raising default permission to the old maximum
makes native reach equal to the old frontier. Enlarging the maximum mask is
a different operation.

The checker is uniform and polynomial-time in the declared input length
$n=\Theta(2^d)$, taking $O(d2^d)$ elementary operations. It is exponential in
the number of Boolean variables; the experiment makes no claim of efficient
verification for succinct arbitrary specifications. A dummy existential move
places this instance at exact depth one.

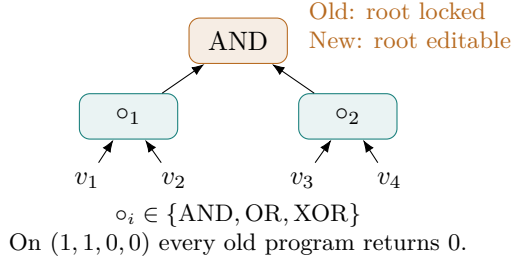
\begin{figure}[t]
\centering
\begin{tikzpicture}[>=Latex,font=\small,
 gate/.style={draw,rounded corners,minimum width=13mm,minimum height=6mm}]
\node[gate,draw=fiAccent,fill=fiAccent!12] (r) at (0,1.5) {AND};
\node[gate,draw=fiFrontier,fill=fiFrontier!9] (l) at (-1.45,.5) {$\circ_1$};
\node[gate,draw=fiFrontier,fill=fiFrontier!9] (q) at (1.45,.5) {$\circ_2$};
\node (a) at (-2,-.35) {$v_1$}; \node (b) at (-.85,-.35) {$v_2$};
\node (c) at (.85,-.35) {$v_3$}; \node (e) at (2,-.35) {$v_4$};
\draw[->] (a)--(l); \draw[->] (b)--(l); \draw[->] (c)--(q); \draw[->] (e)--(q);
\draw[->] (l)--(r); \draw[->] (q)--(r);
\node[anchor=west,text=fiAccent,font=\footnotesize,align=left] at (.8,1.7) {Old: root locked\\New: root editable};
\node[align=center,font=\footnotesize] at (0,-1) {$\circ_i\in\{\mathrm{AND},\mathrm{OR},\mathrm{XOR}\}$\\On $(1,1,0,0)$ every old program returns $0$.};
\end{tikzpicture}
\caption{An enforced code interface and its exclusion witness at $d=4$.
Changing the root to OR and choosing AND below it yields output $1$ on the
displayed input. Counterexample-guided search alone cannot make that patch
admissible under the old permissions.}
\label{fig:cut}
\end{figure}

\begin{proposition}[An immutable-gate exclusion certificate]\label{prop:cut}
For every $d=2^h\ge4$, let
\[
f_*(v)=\left(\bigwedge_{j\le d/2}v_j\right)\lor
          \left(\bigwedge_{j>d/2}v_j\right).
\]
The truth table of $f_*$ belongs to the expanded frontier and not to the old
frontier, regardless of the proposal policy. The same claim holds even if
every old nonroot node is editable among the three allowed operations.
\end{proposition}
The proof is in Appendix~\ref{app:pruning}.
If the expanded stage defaults to $M_{\rm all}$, this is also an explicit
witness in $\Lang_R(S_+)\setminus\Front_R(S)$ for Theorem~\ref{thm:diagnostic}.
Both languages remain within exact polynomial-time verification. This is an
interface separation, not a complexity-class separation.

\subsection{Counterexamples improve search without changing coverage}
The repair agent sees the initial code, editable nodes, candidate language,
and verifier feedback. Its proposal interface does not expose $T$ or the
target program. A verifier call returns either acceptance or the first input
on which the candidate disagrees with $T$, together with its expected value.
The target is private to this software interface, not a cryptographically
isolated service.

We compare two policies using the same target-independent candidate order.
Blind enumeration proposes each candidate in order. The feedback policy
retains counterexamples and skips candidates inconsistent with any retained
example, then calls the verifier on the next consistent one. Both use the
same admission checks, final verifier, and permitted program set. This is a
finite CEGIS loop, following the synthesis literature
\cite{SolarLezama2006,JhaSeshia2017}.

\begin{proposition}[Sound pruning preserves solutions]\label{prop:prune}
For a fixed finite candidate set and order, suppose an exact verifier accepts
if and only if the candidate is correct, and every negative reply supplies a
true counterexample to the current candidate. Consistency
filtering never removes a correct candidate. If a correct candidate exists,
the feedback policy returns the same first correct candidate as blind
enumeration with no more verifier calls. Neither policy changes the stage's
native language or frontier.
\end{proposition}
The proof is in Appendix~\ref{app:pruning}. The last conclusion uses the stated
exposure convention: it admits all permission-valid final records independently
of the online policy. The proposition does not bound total computation. Local
candidate filtering costs work, and the complete candidate space is constructed
in advance in this testbed.

The exact task family permits exhaustive paired evaluation under sixteen fixed candidate orders. The complete protocol and counting conventions are in Appendix~\ref{app:repair-evaluation}. Table~\ref{tab:results} reports coverage and verifier calls; local filtering and precomputation are excluded from that call metric.

\begin{table}[t]
\centering\small
\begin{tabular}{@{}rlrrr@{}}
\toprule
$d$ & Edit mask & $|F|$ & Blind calls & Guided calls\\
\midrule
4 & $M_{\rm leaf}$ & 9 & 5.00 & 2.94\\
4 & $M_{\rm all}$ & 27 & 14.00 & 4.27\\
8 & $M_{\rm leaf}$ & 81 & 41.00 & 5.56\\
8 & $M_{\rm all}$ & 2,187 & 1,094.00 & 10.42\\
\bottomrule
\end{tabular}
\caption{Exact coverage and mean verifier calls to acceptance. Each distinct
reachable target is evaluated under sixteen common candidate orderings.
Default native reach has three tasks at each size; changing the default to
$M_{\rm leaf}$ yields native reach $9$ or $81$ without changing the old frontier.
The four rows contain 36,864 paired runs in total.}
\label{tab:results}
\end{table}
At $d=8$, full permissions admit 2,187 distinct targets. Blind enumeration
averages 1,094 verifier calls, while feedback averages 10.42. Both reach every
one of these targets when allowed to complete. The guided policy makes at most
35 verifier calls in these runs. The old maximum mask admits 81 targets;
its smaller default admits three. These counts are computed from the admitted
program sets independently of either search policy.

The reduction in verifier calls can be large even though the languages are
identical between policies. Enlarging edit permissions changes coverage and
also increases the search space; it need not improve fixed-budget success.
For a given ordering, the feedback policy still inspects candidates locally
up to the same first correct candidate. Thus verifier-call savings should
not be reported as proportional runtime or compute savings.

\begin{figure}[t]
\centering
\begin{tikzpicture}
\begin{axis}[width=.76\linewidth,height=6cm,
 xmode=log,log basis x=2,xmin=1,xmax=64,ymin=0,ymax=105,
 xtick={1,2,4,8,16,32,64},xticklabels={1,2,4,8,16,32,64},
 ytick={0,25,50,75,100},xlabel={Verifier-call budget},
 ylabel={Solved pairs (\%)},grid=major,grid style={gray!15},
 tick label style={font=\footnotesize},label style={font=\small},
 legend style={font=\footnotesize,at={(.03,.97)},anchor=north west,draw=none,legend columns=1}]
\addplot[thick,fiNative,mark=square*] coordinates {(1,0.045725) (2,0.091449) (4,0.182899) (8,0.365798) (16,0.731596) (32,1.463192) (64,2.926383)};
\addlegendentry{Blind enumeration}
\addplot[thick,fiFrontier,mark=*] coordinates {(1,0.045725) (2,0.148605) (4,0.900206) (8,17.175354) (16,97.930956) (32,99.991427) (64,100.000000)};
\addlegendentry{Counterexample feedback}
\end{axis}
\end{tikzpicture}
\caption{Search effectiveness at $d=8$ with the same full-permission frontier.
Percentages use all 2,187 reachable targets and sixteen orderings (34,992
pairs). Curves count verifier calls, excluding local filtering and
precomputation. Coverage is identical; no hierarchy separation is implied.}
\label{fig:search}
\end{figure}
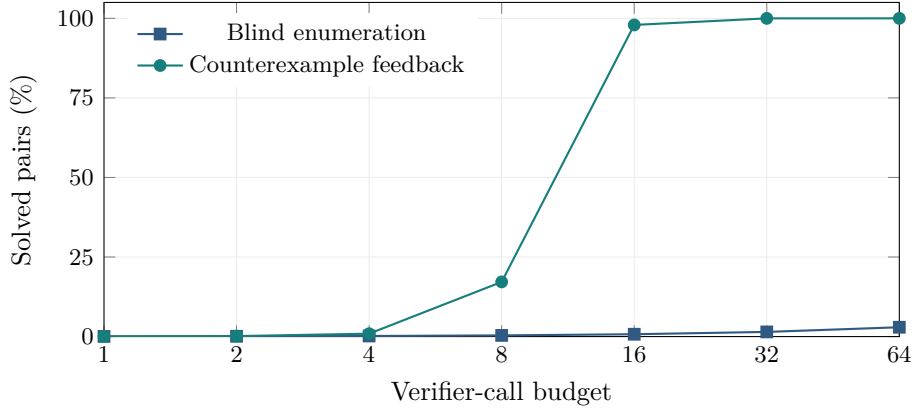

\paragraph{Validation and limits.}
We check the bit-parallel verifier against separately compiled Python
execution on every program-input pair at both sizes. Additional checks cover
permission violations, malformed encodings, certificate substitution, genuine
counterexamples, and unreachable targets. The all-source exclusion in
Proposition~\ref{prop:cut} holds for the growing family; empirical call counts
are only those of the two finite sizes. The task is Boolean-program repair
with fixed syntax and exact feedback. It does not demonstrate learned
improvers, natural-language specifications, autonomous evaluator redesign,
or transfer to software repositories.

\section{Reliability and evaluating proposed successors}\label{sec:random}
\subsection{Reliability under a fixed acceptance contract}
An occasional accepting run is insufficient evidence of capability when a false answer can pass with nonzero probability. The pointwise gap in Definition~\ref{def:stage} distinguishes the high and low sides before the verification moves are quantified. The following result concerns repeated independent checks of the same complete tuple, with every witness domain fixed.

\begin{theorem}[Amplification preserves the frontier]\label{thm:amplify}
Fix a stage $S$ and a polynomial-time computable, odd positive,
polynomially bounded repetition count
$N(n)$. Replace its terminal core by a strict majority of $N(n)$ independent
runs on the \emph{same} tuple. Preserve all object and move domains and the
exposure relation, while allowing polynomially more coins and runtime.
The resulting stage $S^{[N]}$ has the same $D_S$, quantified acceptance,
native reach, and frontier. Its pointwise error is at most $e^{-N/18}$.
\end{theorem}
Independent majority improves reliability without moving any tuple across the acceptance gap. The proof is in Appendix~\ref{app:random-proofs}.

For any constant $0<\delta<1/2$, sufficiently many odd repetitions yield
error at most $\delta$ with the same acceptance language, using threshold
$1-\delta$ for the amplified predicate. The runtime and raw core implementation
change, but semantic authority does not. Preserving the witness domains is
essential: enlarging them while amplifying would be a separate stage change.
Repeatedly sampling \emph{different} candidate answers is a search policy, not
the verifier amplification covered by this theorem. Correlated judge calls
also do not satisfy its independence premise.

The false-tuple counterexample and exact numerical curves are in Appendix~\ref{app:amplification-example}. They distinguish verifier amplification from drawing different proposed answers. A learned judge requires a separate justification of its gap and task-relative soundness; repeated agreement does not supply that justification.

\subsection{A separate contract for successor selection}
A proposed successor can change the solver, harness, evaluator, or improver. Structural comparison asks which tasks its declared stage accepts; empirical selection asks whether it is better under an improvement criterion. Neither a containment proof nor a frontier escape establishes that a benchmark score transfers to deployment. Conversely, a score gain does not by itself prove a changed acceptance language.

Proposition~\ref{prop:selection} controls false selection across an adaptive run under a history-conditional error budget. Proposition~\ref{prop:paired} gives a sufficient implementation using fresh paired evaluations and an independently justified allowance for evaluation-to-deployment mismatch. Their full assumptions, thresholds, proofs, and failure cases are in Appendix~\ref{sec:rsi} and Appendix~\ref{app:paired-proof}. These results supplement the accepted-language account with a separate validity condition; they do not certify a learned judge from the data used to select its successors.

\section{Scope, related work, and conclusion}
The framework compares declared verification families through task soundness, accepted-input languages, and admissible evidence. Its central criterion is a semantic containment statement, and its representation invariant prevents equivalent placements of a check from being treated as different capability. The program-repair example combines this account with an executable feedback policy and a general exclusion witness: faster search preserves coverage, while a permitted new operation admits a target unavailable to all old-interface programs.

Classical verification and hierarchy theory provide the complexity background \cite{CKS1981,Sipser1983,Lautemann1983}. The placement, depth, and certificate-substitution results retain their explicit uniformity and separation assumptions (Appendices~\ref{app:class-boundaries} and~\ref{sec:limits}); they are not new class separations. The repair check enumerates explicit truth tables, so its polynomial cost is in that expanded input length, not in the number of Boolean variables. The synthetic tasks and audits establish no generalization or learned recursive-improvement result for deployed LLM agents. Appendix~\ref{sec:scope} gives the full comparison with prior work and the additional scope limits.

Contemporary systems can revise prompts, workflows, code, evaluators, and improvement procedures. The relevant questions remain which accepted behavior changes, which evidence the old interface already allowed, and whether successor evaluation remains valid. A changed goal requires an explicit task comparison, and any necessity of human input depends on assumptions about its exclusive information source (Appendix~\ref{sec:goal-information}). The framework turns these questions into distinct obligations for a verified capability expansion and for the empirical value of the update.

\section*{AI tool use}
OpenAI's Codex assisted with manuscript revision, literature synthesis,
proof development and checking, and preparation of executable artifacts and
figures. The author is responsible for the paper's content. The scope of the
executable and formal checks is specified in Appendix~\ref{app:reproduce}.

\appendix
\section{Digital computation and verification background}\label{app:background-overview}\label{sec:prelim}
\subsection{Complete bounded-stage contract}\label{app:stage-contract}
\begin{definition}[Bounded probabilistic stage]\label{def:stage}
A stage for a polynomially balanced task $R$ is
$S=(k,p_S,\sigma_0,\Exp_S,\Core_S)$, with the following data and contracts.
\begin{enumerate}[label=(\roman*)]
\item $k\ge1$ is fixed. The polynomial $p_S(n)$ bounds the lengths of support
$\sigma$, transcript $t$, output $z$, and each verification move $y_i$.
Each object type has its own declared polynomial length limit within this
common envelope. Raising the envelope to encode a longer transcript does not
expand the other domains. A move ranges over all binary strings up to its
declared limit.
The default support $\sigma_0(x)$ is uniformly polynomial-time computable
and admissible; we suppress its input argument when unambiguous.
All quantified domains include the empty string.
\item $\Exp_S(x,\sigma,t)$ is a deterministic polynomial-time predicate that
checks admissible support and transcript encodings and their length bounds.
Support supplies data through this interface; it does not install an oracle.
\item $\Core_S(x,z,t,\vec y;r)$ is a uniform polynomial-time Boolean predicate.
Here $r$ is uniform on $\bits^{\rho_S(n)}$ for a declared polynomial $\rho_S$.
Its runtime has a polynomial bound separate from the object-length bound.
Malformed encodings are rejected. The random tape is hidden until after all
moves, so moves cannot depend on it.
\item On every well-formed bounded tuple $w=(x,z,t,\vec y)$, the acceptance
probability obeys
\begin{equation}\label{eq:gap}
a_S(w):=\Pr_r[\Core_S(w;r)=1]\in[0,1/3]\cup[2/3,1].
\end{equation}
Define $D_S(w)=\mathbf1\{a_S(w)\ge2/3\}$, with value zero on malformed or
out-of-bound tuples, and define quantified acceptance
\begin{equation}\label{eq:accepted-output}
\Acc_S(x,z,t)\quad\Longleftrightarrow\quad\Q_k\vec y\ D_S(x,z,t,\vec y)=1.
\end{equation}
\item Task-relative soundness holds: for every exposed transcript,
\begin{equation}\label{eq:sound}
\Exp_S(x,\sigma,t)\ \wedge\ \Acc_S(x,z,t)\quad\Longrightarrow\quad R(x,z).
\end{equation}
\end{enumerate}
The coin and runtime bounds are part of the core specification. Increasing
them alone does not increase the domains of support, output, or moves.
\end{definition}

The tuple retains the exact stage interface while allowing hidden core coins.
The gap makes $D_S$ a total $\BPP$ predicate on encoded tuples: the original
randomized algorithm decides its high-probability set. This observation requires
the gap on \emph{all} bounded tuples, with an efficiently checked encoding
domain. A promise on an unspecified, computationally inaccessible subset would
not justify the same $\BPP$ or derandomization claims.

\paragraph{What soundness requires.}
The gap and soundness are different assumptions. A checker that confidently
accepts a false answer can satisfy the gap and violate soundness. Conversely,
at depth two a false output can have a locally accepting move and still be
rejected by the quantified game: $D(a,u)=\mathbf1\{u=0\}$ satisfies the gap,
but $\exists a\,\forall u\,D(a,u)$ is false on the Boolean challenge domain.
Thus quantified soundness is \emph{not} equivalent to requiring low acceptance
on every move tuple of every false output. The latter is sufficient and stronger.
At depth one it is equivalent on exposed transcripts.

In a patch audit, $\sigma$ can encode a counterexample, $t$ the bounded
proposal/test/repair record, and the final $z$ the patch to certify. Any
verification moves after $z$ support or challenge that fixed patch. If the
output itself must change in response to a universal challenge, a different
game with an output strategy is required; it is not silently encoded as a fixed
output in~\eqref{eq:accepted-output}.

\subsection{Neural models, programs, and verification}\label{sec:computational-bridge}
An agentic system can use a neural model to propose actions while software
maintains memory, invokes tools, and checks results. ReAct, for example,
interleaves language-model reasoning with actions that obtain information
from an external environment~\cite{Yao2023ReAct}. This extends the unit of
analysis from an isolated model to an interacting system; neural inference
remains a component of that system. The present framework studies its
verification interface rather than a particular network architecture.

A Turing machine supplies a common description of such digital computation.
Stored weights, tokens, memory, and code have finite bit encodings. A forward
pass executes arithmetic with finite precision, and decoding and tool control
execute further instructions. Once the random bits and external responses
are specified, these operations define a computational execution. Standard
Turing-machine and RAM models simulate one another with polynomial overhead
under their stated cost conventions~\cite{BarakRunningTime}. No claim that a
standalone network is itself a universal machine is needed. Nor does digital
computability alone imply polynomial time: inference, precision, memory,
and tool costs need explicit bounds. Appendix~\ref{app:digital-resources}
details this accounting and the distinction between a fixed checkpoint and
a uniformly specified family.

For intuition, consider a coding agent asked to repair a program. A neural
model suggests a patch; a controller runs permitted checks and returns
feedback; a terminal rule determines whether to accept the final patch.
Table~\ref{tab:agent-mapping} maps these roles to the paper's notation.
The correctness relation must match the claim: passing a finite test suite
certifies that suite's condition, not arbitrary program correctness.

\begin{table}[htbp]
\centering\small
\begin{tabularx}{\linewidth}{@{}p{.29\linewidth}X@{}}
\toprule
Coding-agent component & Role in the framework\\
\midrule
Program and specification & Task input $x$; $R(x,z)$ states the required correctness of patch $z$.\\
Neural proposer and decoder & Search policy producing candidate patches and records; its discovery probability is separate from acceptance.\\
Permitted feedback and memory & Support $\sigma$ and a bounded record $t$, when admitted by the declared interface.\\
Tool and editing permissions & Exposure predicate $\Exp_S$ checks which support and records are admissible.\\
Final verification & $\Core_S$ checks a complete tuple; $\Q_k$ states the required proposal/challenge conditions.\\
\bottomrule
\end{tabularx}
\caption{A role mapping for a declared coding-agent protocol. These roles can
have neural or conventional implementations; membership in the model still
requires the bounds, probability gap, and soundness in Definition~\ref{def:stage}.}
\label{tab:agent-mapping}
\end{table}

With exposure and acceptance fixed, changing neural weights or decoding can
improve how often an accepted witness is found while preserving the accepted
task sets. If an update changes admissible records or the checking rule,
those changes instead enter the stage comparison explicitly. A neural model
can also serve as a critic or checker, but its implementation does not
establish the required soundness or probability gap. Thus architecture is
abstracted from the notation, not excluded from the system.

The complexity results classify the declared \emph{acceptance languages}.
They do not predict a model's benchmark score, training sample efficiency,
or GPU latency. In particular, $\exists$ denotes the existence of a suitable
proposal, and $\forall$ demands every challenge in the declared domain.
Sampling a neural proposal or critic is an execution policy, not a way to
discharge those quantifiers automatically. This distinction also explains
why finite experiments illustrate mechanisms without proving asymptotic
class separations.

\subsection{Task encodings and classical verification}
Inputs, outputs, and messages are finite binary strings under fixed encodings.
Write $n=|x|$. A task relation $R\subseteq X\times Z$ is \emph{polynomially
balanced} if there is a polynomial $q$ such that
$R(x,z)\Rightarrow |z|\le q(|x|)$. This bounds every valid encoded output.
It does not assert that $R$ is decidable in polynomial time, or that a valid
output can be efficiently found. For a bounded patching task, $x$ contains the
program and specification, and $z$ is a patch of the declared maximum length.

\paragraph{Certificates and alternating verification (classical background).}
$\Pclass$ is deterministic polynomial time. A language is in $\NP$ if membership
has a polynomial-length certificate checked in polynomial time. More generally,
for fixed $k\ge1$, write
\[
\Q_k\vec y=\exists y_1\,\forall y_2\,\exists y_3\cdots Q_k y_k,
\]
with each string ranging over a nonempty polynomially bounded domain. A language
$L$ is in $\SP{k}$ when $x\in L$ is equivalent to $\Q_k\vec y\,V(x,\vec y)=1$
for a deterministic polynomial-time predicate $V$. The universal-start version
defines $\PP{k}$, and $\PH=\bigcup_{k\ge0}\SP{k}$ with $\SP{0}=\Pclass$.
These are standard bounded-alternation characterizations~\cite{CKS1981}.
There are $k$ quantifier blocks and $k-1$ switches between quantifier types;
``depth'' below counts blocks.

For example, $\exists a\,\forall u\,\exists v\,V(x,a,u,v)$ says that some
proposal $a$ has a valid reply $v$ for every challenge $u$. The reply may depend
on $u$, while $a$ cannot. A sampled critic does not implement $\forall u$.
An exhaustive search over the bounded domain, or a separately justified
certificate for the same universal condition, is needed. A polynomial bound on
the \emph{length} of $u$ permits exponentially many challenges. Thus a
polynomial-time terminal predicate does not make the whole quantified game a
polynomial-time executable agent.

Existential start matches proposal-first verification. Challenge-first games
have the universal-start dual. Several proposals without an intervening
universal block can be concatenated into one existential string if the total
length remains polynomial. This is the formal reason that additional sampled
reasoning traces need not increase depth.

\paragraph{Randomized terminal predicates.}
$\BPP$ consists of languages decidable in polynomial time with two-sided error
at most $1/3$ on every input. We use it for a \emph{tuple-level} predicate, after
the verification moves have been specified. The classical containment
\[
\BPP\subseteq\SP{2}\cap\PP{2}
\]
is the Sipser--G\'acs--Lautemann bound~\cite{Sipser1983,Lautemann1983}.
These background facts will locate the stage languages. They do not identify
the stage declaration of an observed product.

\section{Verifier classes and classical boundaries}\label{app:class-boundaries}
\begin{definition}[Randomized-verifier classes]\label{def:rv}
For $k\ge1$, a language $L$ belongs to $\SR{k}$ if some polynomial-time
randomized predicate $V(x,\vec y;r)$ has a pointwise gap on all bounded tuples
and
\[
x\in L\quad\Longleftrightarrow\quad
\Q_k\vec y\ [\Pr_r(V(x,\vec y;r)=1)\ge2/3].
\]
$\PR{k}$ is the universal-start dual. On no-instances the gap gives the dual
quantifier formula with threshold at most $1/3$ by negation.
\end{definition}

This is a paper-specific notation for bounded alternation over a total
$\BPP$ predicate. At depth one, $\NP\subseteq\SR{1}\subseteq\MA$:
the witness is chosen before the coins. The additional pointwise restriction
applies even to losing witnesses on yes-instances, so equality with standard
$\MA$ is not assumed. Placing the probability outside the existential witness
instead gives public-coin semantics; it changes the information available to
the witness and must be modeled separately~\cite{BabaiMoran1988}.

\begin{proposition}[Placement and deterministic recovery]\label{prop:placement}
For a sound stage of depth $k$,
\[
\Lang_R(S)\subseteq\Front_R(S)\subseteq\mathrm{dom}(R),\qquad
\Lang_R(S),\Front_R(S)\in\SR{k}.
\]
An exact stage has $\rho_S=0$, hence $D_S=\Core_S$. Its native and frontier
languages belong to $\SP{k}$, and its exact and probabilistic readings coincide.
\end{proposition}
Permitted support can be selected with the first existential move; exact checking removes only the terminal randomness. The proof is in Appendix~\ref{app:placement}.

\begin{theorem}[Expressive completeness of native/frontier pairs]\label{thm:complete}
Fix $k\ge1$ and $L_0\subseteq L_1$ with both languages in $\SR{k}$.
There is a sound depth-$k$ stage with $\Lang_R(S)=L_0$ and $\Front_R(S)=L_1$
for a polynomially balanced task $R$. If both languages are in $\SP{k}$,
the stage can be exact.
\end{theorem}
The proof is in Appendix~\ref{app:placement}.

This is a representation theorem: the interface loses no pair of languages in
the declared class. It is not a procedure for learning a verifier or discovering
the exposure sets of a product. Together with Proposition~\ref{prop:placement},
it supplies the deterministic foundation without reliance on a separate paper.

\subsection{Classical placement and conditional comparisons}
\begin{theorem}[Sandwich bound]\label{thm:sandwich}
For every $k\ge1$,
\[
\SP{k}\subseteq\SR{k}\subseteq\SP{k+1},\qquad
\PP{k}\subseteq\PR{k}\subseteq\PP{k+1}.
\]
\end{theorem}
Replacing the terminal randomized predicate by its classical alternating representation costs at most one extra block. The proof is in Appendix~\ref{app:hierarchy-proofs}.

The bound concerns \emph{families of expressible languages}. It does not assign
an ``exact frontier'' to an arbitrary randomized stage by changing a label.
No general nested-ring comparison of exact and randomized frontiers of the
same physical checker follows from it.

\begin{theorem}[Conditional enlargement and derandomization]\label{thm:regimes}
\begin{enumerate}[label=(\roman*)]
\item If $\SR{1}\not\subseteq\NP$, some sound randomized depth-one stage
has a frontier that no exact depth-one stage can equal.
\item If $\BPP=\Pclass$, then $\SR{k}=\SP{k}$ and $\PR{k}=\PP{k}$ for
all $k$. Every randomized stage has a sound exact companion with the same
depth, object domains, exposure, native reach, and frontier, allowing a
different polynomial bound on terminal runtime.
\end{enumerate}
\end{theorem}
The proof is in Appendix~\ref{app:hierarchy-proofs}.

The enlargement premise is specific to the pointwise-gap class; $\MA\ne\NP$
alone does not supply it. Failure of $\BPP=\Pclass$ would not by itself prove
strict stage-frontier enlargement either. These are conditional comparisons,
not demonstrated advantages of a deployed neural judge.

\section{Depth limits and certificate substitution}\label{sec:limits}
\begin{theorem}[Conditional depth obstructions]\label{thm:depth}
Fix $k\ge1$.
\begin{enumerate}[label=(\roman*)]
\item If $\SP{k}\ne\SP{k+1}$, there is a task $R$ and an exact depth-$(k+1)$
stage $T$ whose native language is not contained in the frontier of any sound
exact depth-$k$ stage for $R$.
\item The analogous assertion holds for randomized depths $k,k+1$ if
$\SR{k}\ne\SR{k+1}$.
\item If $\SP{k+1}\ne\SP{k+2}$, some exact depth-$(k+2)$ target cannot be
covered by any sound randomized depth-$k$ stage for its task.
\end{enumerate}
\end{theorem}
The proof is in Appendix~\ref{app:depth-proofs}.

Soundness is essential to the obstruction: a language outside a class can be a
subset of a language inside that class. Nonmembership alone would not rule
out coverage. Nor does this theorem say that every extra block improves every
stage. Its targets are chosen separator families. Iterating it supplies
different tasks at successive levels, not a training trajectory on one fixed
benchmark. Under the stated assumptions, randomization at depth $k$ cannot
uniformly replace all exact depth-$(k+2)$ protocols. The theorem does not
establish a sharp adjacent-level bound.

\begin{proposition}[When certificates can replace challenge]\label{prop:certificate}
Let $V(x,a,u)$ be polynomial-time with bounded strings $a,u$, and let
$G(x,a)$ be a deterministic polynomial-time guard for the supported family.
Suppose a polynomial-time exact checker $C(x,a,\pi)$ has polynomially bounded
certificates and, for every bounded $(x,a)$ with $G(x,a)=1$, satisfies
\[
\exists\pi\ C(x,a,\pi)=1\quad\Longleftrightarrow\quad
\forall u\ V(x,a,u)=1.
\]
Then the guarded language
\begin{align*}
&\{x:\exists a\ [G(x,a)\wedge\forall u\ V(x,a,u)=1]\}\\
&\qquad=\{x:\exists a,\pi\ [G(x,a)\wedge C(x,a,\pi)=1]\}
\end{align*}
has an exact depth-one certificate representation. Soundness of $C$ alone
gives containment, not equality. If this guarded language is
$\SP{2}$-complete, such a uniform sound-and-complete certificate system
implies $\SP{2}=\NP$.
\end{proposition}
The proof is in Appendix~\ref{app:depth-proofs}.

This is a design constraint derived from classical complexity, with two
separate engineering obligations: soundness of the checker and existence of
short certificates for every accepted proposal. Restricted families can
satisfy both. For example, equality of two XOR circuits can be checked by
computing their coefficient vectors over $\mathbb F_2$, even though direct
universal testing ranges over exponentially many Boolean inputs. Such a
symbolic checker changes the terminal implementation and can remove explicit
challenge on that family. Sampled criticism alone supplies neither obligation.

\subsection{Bounded certificates and the cost of checking}
\label{sec:certificate-resources}

Proposition~\ref{prop:certificate} has three distinct obligations: a
certificate characterizes the same universal claim, every true guarded
claim has a certificate within the declared length bound, and one uniform
algorithm checks those certificates within the declared time bound.
Completeness with arbitrarily long certificates does not establish the
second obligation. A separate polynomial-time checker for each fixed
input does not establish the third.

To make the distinction explicit, fix bounds $p,b,c$ and define
\begin{align*}
L_\forall={}&\{x:\exists a,\ |a|\le p(n),\ G(x,a),\\
&\hspace{35mm}\forall u,\ |u|\le b(n)\Rightarrow V(x,a,u)\},\\
L_C={}&\{x:\exists a,\pi,\ |a|\le p(n),\ |\pi|\le c(n),\\
&\hspace{35mm}G(x,a)\wedge C(x,a,\pi)=1\},
\end{align*}
where $n=|x|$. Bounded soundness gives $L_C\subseteq L_\forall$;
bounded completeness gives the reverse inclusion. These are statements
about languages even before a computational cost is assigned. Polynomial
bounds and a uniform polynomial-time implementation supply the separate
$\NP$ representation.

\begin{proposition}[An explicit circuit resource account]
\label{prop:circuit-cost}
Represent a Boolean checker as a formula tree with input, constant,
negation, and conjunction nodes. A strict evaluator visiting both
children at each conjunction performs exactly one node visit per tree
node. If guard and checker trees have sizes $s_G(n)$ and $s_C(n)$, their
conjunction uses exactly $s_G(n)+s_C(n)+1$ node visits. Consequently,
polynomially bounded tree sizes imply polynomially many visits. If one
uniform polynomial-time procedure constructs these trees and their input
wiring from the bounded strings $(x,a,\pi)$, polynomial certificate and
proposal bounds give a uniform polynomial-time verifier for $L_C$.
\end{proposition}

\begin{proof}
For an input or constant the visit count is one. Negation adds one to
its child's count; conjunction adds one to both children's counts.
Structural induction proves that the count equals tree size, and the
conjunction formula follows immediately. For the last statement, include
the construction time and input access in the cost. A sequential
implementation can store the input assignment and access each requested
bit by a scan of its polynomially bounded encoding; multiplying its
length by a polynomial number of visits remains polynomial. Domain and
wiring checks are included in the same uniform construction.
\end{proof}

The circuit statement supplies an explicit sufficient implementation
contract. It does not assert that every polynomial-time verifier has a
polynomial-size formula tree: sharing can make circuits smaller than
expanded trees. The Lean artifact proves bounded semantic replacement
with explicit string-length restrictions and the exact node-visit
recurrence for this evaluator. Its circuit-family theorem assumes a
stated polynomial size envelope. Uniform tree generation, bit-access
cost, and the complexity-class conclusion use the additional mathematical
premises above; they are not inferred from a bare circuit family or a
reported runtime counter. This separates the verified evaluator from
the hypotheses needed to apply it to a deployed checker.

\section{A reproducible audit, from exact to probabilistic}\label{sec:audit}
\subsection{Running task and three exact designs}
Let $x_\star$ be the task of repairing a lookup function on the three-element
array $(\texttt{alpha},\texttt{beta},\texttt{gamma})$. The declared input
domain for testing is $I=\{-3,-2,-1,0,1,2,3,4\}$. Correct behavior returns
the indexed element for $0\le i<3$ and a default value otherwise.
The candidate set contains the original buggy program, an upper-bound-only
repair, and a full lower-and-upper-bound repair. The task relation is
\[
R(x_\star,z)\iff\forall i\in I,\quad z(i)=\mathrm{spec}(i).
\]
This example is finite and completely enumerated. It does not establish a
complexity separation for general program repair.

The sampled critic checks only $\{0,2,3,4\}$. The upper-bound-only repair
passes all four, but at $i=-1$ returns \texttt{gamma} rather than the default.
The exhaustive design checks all eight inputs. The certificate design checks
a complete table of outcomes against both the candidate execution and the
specification, bound to the declared task and candidate. Both exact designs
reject the incomplete patch and accept the full repair. The sampled checker
is deterministic as a program, but is not sound for $R$: deterministic testing
does not imply exact verification of a broader specification.

\begin{table}[htbp]
\centering\small
\begin{tabularx}{\linewidth}{@{}Xccc@{}}
\toprule
Terminal rule & Incomplete patch & Full repair & Sound for $R$\\
\midrule
Four declared tests & Accept & Accept & No\\
All eight declared inputs & Reject & Accept & Yes\\
Task-bound complete certificate & Reject & Accept & Yes\\
\bottomrule
\end{tabularx}
\caption{Exact enumeration on the bounded lookup task. Acceptance is relative
to the declared specification and domain. The sampled design's apparent reach
cannot be used as a sound frontier for this task.}\label{tab:patch}
\end{table}

The existing executable declaration includes default and optional certificate
transcripts. Default support in the certificate stage supplies no valid
certificate, while permitted optional support exposes the complete correct
record. It therefore has
\[
\Lang_R(S_{\rm cert})=\varnothing,\qquad
\Front_R(S_{\rm cert})=\{x_\star\}.
\]
The exhaustive stage exposes the valid repair natively and has both languages
equal to $\{x_\star\}$. These are sets of the single task input; the accepted
patch projection, reported separately by the artifact, contains only the
full repair. They must not be confused with one another.

To isolate exposure changes, keep the certificate core fixed and compare
three declarations. $S_0$ permits only the default uncertified record under
all support states, so both languages are empty. $S_1$ admits the valid
certificate through optional support, so native reach is empty and the
frontier is $\{x_\star\}$. $S_2$ makes that record native, so both languages
are $\{x_\star\}$. Thus $S_0\to S_1$ expands support and the frontier,
and $S_1\to S_2$ internalizes already permitted support. The direct
$S_0\to S_2$ comparison exhibits native escape under unchanged authority.
Repeating default attempts in $S_1$ cannot introduce the excluded native
record; searching its optional support can find a record already in its
frontier. This is the finite distinction between search effort, support
expansion, and internalization.

\subsection{A controlled stochastic extension}
For the same finite checker $C$, reject malformed and out-of-bound records
deterministically. On well-formed bounded tuples, define a noisy terminal
core that uses two fair bits: with probability $3/4$ it returns $C$ and with probability $1/4$
it returns $1-C$. Accepted exact tuples now have probability $3/4$ and
rejected tuples probability $1/4$. The gap holds, $D=C$, and soundness and
every native/frontier set are unchanged. Three independent evaluations with
majority change these probabilities to $27/32$ and $5/32$ without changing
the sets. Existential acceptance over the two-bit tape accepts every such tuple,
including the incomplete patch.

This deliberately controlled example verifies the semantic distinction using
exact rational arithmetic; it is not an empirical model of a learned judge.
It can also expose misleading error reductions. Accepting if \emph{any} of
three calls accepts gives a false acceptance probability
$1-(3/4)^3=37/64$, which violates the gap on a false tuple. The majority
rule, independence, and fixed candidate are therefore substantive assumptions.

\subsection{Analytic search and interface controls}
Appendix~\ref{sec:xor-case} gives a growing XOR-synthesis family whose exact
frontier is determined by a quota on referenced variables. Two explicit
proposal policies have unbounded ratios of per-attempt success while sharing
the same accepted languages. The target coefficients are provided in the
input, so this is an analytic control rather than a learning experiment.
Section~\ref{sec:repair} supplies a complementary executable agent whose
proposal interface receives counterexamples instead of the target table.

\subsection{Declaration protocol and disagreement rules}\label{sec:declaration}
A reproducible audit first records the following evidence, before assigning
complexity labels or claiming an escape.
\begin{enumerate}
\item \textbf{Task and encoding.} Specify input size, output meaning, the
correctness relation, and any restricted domain. Separate test-suite
correctness from full program correctness.
\item \textbf{Stage and budgets.} Version the model, decoder, interface,
default support, optional support, transcript grammar, object-length bounds,
coin policy, and runtime bounds. Trace permissibility needs an executable
predicate or a proof; one observed trace establishes only one possibility.
\item \textbf{Acceptance contract.} Identify the final checker and provide
task-soundness evidence. A randomized checker additionally needs a uniform
gap argument. Average agreement is insufficient. Universal moves need
exhaustive bounded semantics or an equivalent proved procedure.
\item \textbf{Comparison.} Align task and transcript encodings. Prove
authority/exposure containment, or exhibit a counterexample to it. A frontier
escape claim additionally needs an accepted target input and evidence excluding
\emph{all} source-frontier witnesses for that input.
\item \textbf{Audit status.} Record established conclusions relative to
validated premises, candidate claims awaiting specific evidence, and
indeterminate claims with missing or disputed premises. Preserve competing
declarations rather than voting one into truth.
\end{enumerate}

A schema linter can reproduce the consequences of identical evidence records.
It cannot validate an arbitrary semantic inclusion from a certificate string
or make two auditors agree about an opaque deployment. A false output proved to satisfy $\Acc_S$ on an exposed transcript
falsifies task soundness. One observed false stochastic acceptance does not:
it may be a permitted low-probability event. At depth greater than one,
locally high acceptance also need not imply quantified acceptance. A proved
admissible tuple with probability in $(1/3,2/3)$ falsifies the gap; a supposedly universal critic that checks only
a sample invalidates that depth declaration. Finite measurements can reveal
counterexamples but do not in general prove uniform contracts on an infinite
family. An unobserved trace is not evidence that it is forbidden.

The patching artifact supplies a content-addressed declaration, exact outcomes,
accepted input languages, and separate output projections. The additional
probabilistic artifact enumerates random tapes and uses rational probabilities.
The patching artifact also distinguishes inclusion of visible facts from inclusion of actual
transcripts: exhaustive-test facts can be contained in certificate-stage facts
even though their encoded transcript sets are disjoint. Only the latter sets,
or a proved translation between them, enter Theorem~\ref{thm:diagnostic}.
Appendix~\ref{app:reproduce} specifies how to reproduce the artifacts. We do not use the
names of commercial systems or multiagent roles to assign established
$\SP{k}$ or $\SR{k}$ membership to their implementations.

\section{Recursive self-improvement: agent changes and selection}\label{sec:rsi}
The framework separates two questions about self-improvement: how a successor
changes the tasks its protocol accepts, and how reliably the selection process
identifies an improvement. The first concerns verification structure; the
second requires a separate error contract for the improvement criterion.

Duan et al. distinguish structural recursion, where an inherited revision
controls a later improvement round, from effective recursion, where it produces
better successors under comparable budgets and independent assessment
\cite{Duan2026LastAI}. Both can occur within one verification class. Our
comparison concerns the accepted languages and contracts of those successive
systems; it does not replace the empirical test of improvement effectiveness.

\subsection{What recent systems make observable}
Self-improvement can change an answer within one task, a reusable agent
component across tasks, or the mechanism that generates subsequent changes.
These are different comparison scopes. GEPA evolves prompts from execution
feedback, ACE maintains structured contextual playbooks, and AFlow searches
code-represented workflows~\cite{Agrawal2026GEPA,Zhang2026ACE,Zhang2025AFlow}.
RSIAgent accumulates memory without parameter updates and freezes it for
subsequent evaluation~\cite{Zhu2026}. Such mechanisms motivate explicit
support and transcript declarations, including what the old interface already
permitted.

Executable components can also evolve. DGM searches an archive of modified
agent programs; AHE connects harness edits to component-level observations;
Self-Harness uses failure analysis and regression testing to select harness
changes~\cite{Zhang2026DGM,Lin2026AHE,Zhang2026SelfHarness}.
At the meta level, STOP applies a program improver to its own code while
leaving the underlying language model unchanged, and Hyperagents makes both
the task agent and meta agent editable~\cite{Zelikman2024STOP,Zhang2026Hyperagents}.
AIDE$^2$ reports successive research-agent improvements and held-out transfer,
while noting that its outer improver's advantage over a strong baseline is
not decisively resolved~\cite{Srikanth2026}.

These reports make it necessary to record \emph{which component changes and
which evaluation remains fixed}. They provide empirical evidence within their
experimental settings, rather than stage-contract or hierarchy-separation
proofs. Appendix~\ref{app:rsi-literature} compares the mechanisms and the evidence
needed to connect them to the present model. The comparison includes the
September 22 version of the RSI survey~\cite{Duan2026LastAI} and studies
of evolving evaluators available through October 9, 2026.

\subsection{Deployed stage, improver, and selection rule}
Write a proposed transition as
\[
\widehat S_{i+1}=U_i(S_i,h_i;\eta_i),\qquad
S_{i+1}=J_i(S_i,\widehat S_{i+1},h_i),
\]
where $U_i$ is an improver, $h_i$ is its bounded interaction history,
$\eta_i$ is its proposal randomness, and $J_i$ is the selection rule.
The rule can retain the old stage. A recursive loop lets an accepted change
affect a subsequent improver or its inputs. It need not add a universal
verification block. The task relation $R$ used to evaluate deployed outputs
must also be distinguished from the improvement criterion used by $J_i$.
A higher benchmark score is neither a proof of
$\Lang_R(S_i)\subseteq\Lang_R(S_{i+1})$ nor an escape witness.

Figure~\ref{fig:rsi-loop} separates the update loop from the two contracts
used to analyze it: accepted-language structure and false-selection risk.
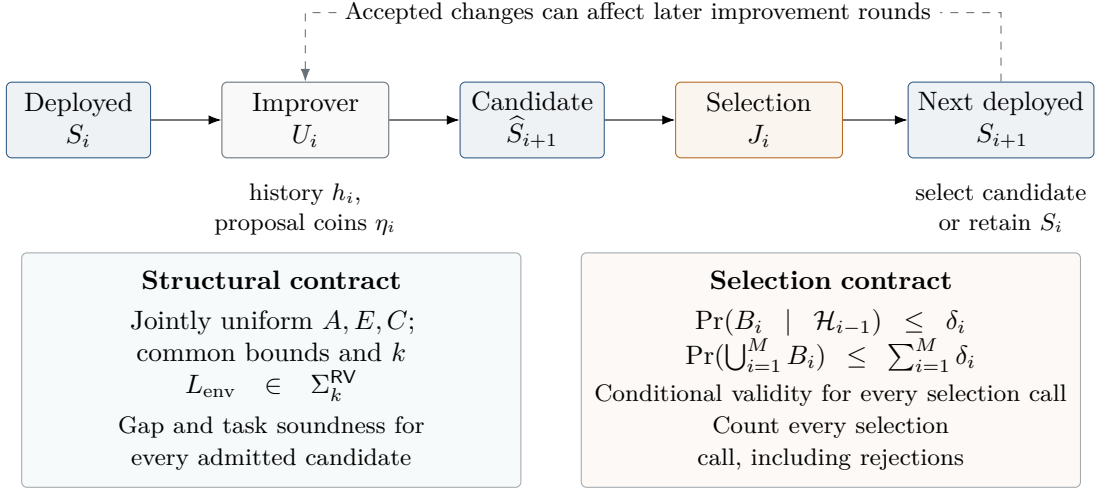
\begin{figure}[tbp]
\centering
\begin{tikzpicture}[x=1cm,y=1cm,>=Latex,font=\small,
 state/.style={draw=fiNative,fill=fiNative!8,rounded corners=2pt,
   text width=1.65cm,minimum height=1cm,align=center},
 process/.style={draw=fiGray,fill=fiGray!4,rounded corners=2pt,
   text width=1.95cm,minimum height=1cm,align=center},
 contract/.style={draw=fiGray!60,rounded corners=2pt,align=center,
   text width=6.25cm,minimum height=3.1cm,inner sep=5pt}]
\node[state] (s) at (1,2.3) {Deployed\\$S_i$};
\node[process] (u) at (4,2.3) {Improver\\$U_i$};
\node[state] (candidate) at (7,2.3) {Candidate\\$\widehat S_{i+1}$};
\node[process,draw=fiAccent,fill=fiAccent!7] (j) at (10,2.3) {Selection\\$J_i$};
\node[state,text width=2.2cm] (next) at (13.2,2.3) {Next deployed\\$S_{i+1}$};
\draw[->] (s)--(u);
\draw[->] (u)--(candidate);
\draw[->] (candidate)--(j);
\draw[->] (j)--(next);
\node[below=2mm of u,font=\footnotesize,align=center] {history $h_i$,\\proposal coins $\eta_i$};
\node[below=2mm of next,font=\footnotesize,align=center] {select candidate\\or retain $S_i$};
\draw[->,dashed,fiGray] (next.north)--(13.2,3.72)--(4,3.72)--(u.north);
\node[fill=white,inner sep=3pt,font=\footnotesize] at (8.4,3.72)
  {Accepted changes can affect later improvement rounds};
\node[contract,fill=fiFrontier!4,anchor=north] (struct) at (3.55,.55)
  {\textbf{Structural contract}\\[3pt]
   Jointly uniform $A,E,C$; common bounds and $k$\\
   $L_{\rm env}\in\SR{k}$\\[3pt]
   \footnotesize Gap and task soundness for every admitted candidate};
\node[contract,fill=fiAccent!4,anchor=north] (select) at (10.95,.55)
  {\textbf{Selection contract}\\[3pt]
   $\Pr(B_i\mid\mathcal H_{i-1})\le\delta_i$\\
   $\Pr(\bigcup_{i=1}^{M}B_i)\le\sum_{i=1}^{M}\delta_i$\\[3pt]
   \footnotesize Conditional validity for every selection call\\
   Count every selection call, including rejections};
\end{tikzpicture}
\caption{An RSI loop and two distinct guarantees. The dashed feedback indicates
that an accepted revision can alter later improvement, including the
improver's inputs or inherited mechanism. It does not add a quantifier block.
Under Proposition~\ref{prop:meta}'s fixed admission, exposure, interpretation,
and common bounds, admitted histories define a language in the same
verification class. Proposition~\ref{prop:selection} separately controls
false selection relative to an improvement criterion. Neither guarantee alone
establishes that successors are empirically better.}
\label{fig:rsi-loop}
\end{figure}

Changes to weights, prompts, tools, or memory identify candidate stage
changes; their semantic effect depends on the declared interface. A better
proposal policy within an unchanged admissible transcript family can improve
observed performance without changing the frontier. An added tool can change
exposure or the terminal predicate. Internalizing previously optional memory
can enlarge native reach inside an unchanged frontier. None of these effects
is determined by the word ``recursive.'' Uniform complexity claims further
require a uniformly described family as input size grows; unrelated
checkpoints at each size must not silently supply nonuniform advice.

\begin{proposition}[Bounded self-modification inside a fixed envelope]
\label{prop:meta}
Let $h,c$ be polynomially bounded history and candidate strings. Assume one
deterministic polynomial-time admission predicate $A(x,h,c)$, one such
exposure predicate $E(x,h,c,\sigma,t)$, and one uniform polynomial-time
randomized terminal interpreter $C(x,h,c,z,t,\vec y;r)$. These algorithms
include all encoding and object-domain checks. There are common polynomial
object, coin, and runtime bounds and the same fixed $k$ move blocks, whose
domains are identical for every candidate. The
default support is uniformly polynomial-time computable from $(x,h,c)$.
Every admitted candidate $S_{h,c}$ satisfies the pointwise gap and quantified
soundness for a common $R$. Choose $h,c$ before the moves and hidden coins.
Then
\[
L_{\rm env}=\{x:\exists h,c\ [A(x,h,c)\wedge
x\in\Front_R(S_{h,c})]\}\in\SR{k}.
\]
The same holds with native reach in place of the frontier. With exact
terminal checks both languages belong to $\SP{k}$.
\end{proposition}
A bounded candidate program and history fit within the first existential choice when the verification interface is fixed. The proof is in Appendix~\ref{app:rsi-proofs}.

This is a closure observation under explicit uniform bounds, not a claim that
arbitrary self-modifying software has a sound admission test. Establishing
the gap and soundness for every permitted candidate is an additional proof
obligation; a syntax check alone does not establish it. The proposition also
concerns existence of an accepted history, not the probability or cost of
discovering it. A successor can escape the particular deployed $S_i$ while
remaining within this larger, fixed enclosing stage. Conversely, a run that
changes the envelope, interpreter, or quantifier protocol requires a new
declaration. Polynomial time for each separately fixed candidate is insufficient:
the admission, exposure, and interpretation must be jointly uniform in its code.
The admitted envelope can contain histories the actual improver never discovers.
The result does not cover an unbounded improvement process with
no common resource bound.

Table~\ref{tab:update-contract} applies these obligations to the XOR family.
Here exposure is what the harness permits. A policy need not realize every
permitted record. If a different declaration instead defines exposure through
policy-attainable records, policy support becomes part of that declaration
and must be compared explicitly.

\begin{table}[ht]
\centering\small
\begin{tabularx}{\linewidth}{@{}p{.22\linewidth}X X@{}}
\toprule
Permitted edit & Effect in the XOR family & Contract to retain or establish\\
\midrule
Full-support proposal policy & Changes discovery probability; $L,F$ unchanged & Same exposure, output binding, and exact core.\\
Default quota $0\to b$ & Sets $L$ equal to the existing $F$ & Both defaults belong to the same admitted quota range.\\
Maximum quota $b\to b+1$ & $F$ gains weight-$(b+1)$ targets & New exposure bound and an all-source-witness exclusion proof; core remains sound.\\
Replace the checker & Depends on its accepted behavior & Re-establish gap and task soundness; syntax admission is insufficient.\\
\bottomrule
\end{tabularx}
\caption{A concrete update contract. All quota languages here remain in
$\Pclass$; these edits are not claims of learned recursive improvement.}
\label{tab:update-contract}
\end{table}

\subsection{Selection error across an adaptive improvement run}
SIFT uses pairwise learned-judge comparisons to prioritize a self-improvement
search, with expensive task evaluations reserved for selected candidates~\cite{Fu2026SIFT}.
Its ranking signal has a different role from a sound terminal verifier.
AI4AI-Bench instead evaluates training-algorithm changes through fresh runs
under a hidden fixed evaluator~\cite{Chi2026AI4AI}.
Both make the distinction between search feedback and final acceptance
operationally important.

Probabilistic acceptance needs a second accounting level when many candidate
versions are compared. Let $B_i$ be the event that the $i$th selection call
accepts a candidate that violates a separately specified improvement
criterion. Let $\mathcal H_{i-1}$ contain the available history and selected
candidate just before that call's fresh evaluation randomness.

\begin{proposition}[Finite adaptive selection budget]\label{prop:selection}
Suppose a run makes at most $M$ selection calls, including rejected candidates.
If for every reachable pre-call history
\[
\Pr(B_i\mid\mathcal H_{i-1})\le\delta_i
\]
for fixed nonnegative budgets $\delta_i$, then
\[
\Pr\!\left(\bigcup_{i=1}^{M}B_i\right)
\le\sum_{i=1}^{M}\delta_i.
\]
In particular, if each invalid fixed candidate has per-trial false acceptance
probability at most $1/3$, evaluated with independent fresh coins after the
candidate is fixed, an odd majority size
$N\ge18\log(M/\delta)$ gives a total false-selection probability at most
$\delta$ for $0<\delta<1$.
\end{proposition}
The conditional guarantee must cover each selection call; summing these budgets controls the risk of any false selection. The proof is in Appendix~\ref{app:rsi-proofs}.

This is a standard probability bound applied to stage selection; independence
between different calls is unnecessary when the conditional bound holds.
Its premise concerns the improvement criterion and is not supplied merely by
the tuple-level gap of a deployed checker. A pointwise error bound for invalid
candidates under the selection rule must be established separately. Good
average judge accuracy does not supply it. Nor do fresh decoding coins repair
adaptive overfitting to a repeatedly reused benchmark; the evaluation data
and feedback policy require their own validity argument~\cite{Dwork2015}.
The finite audit's three-call example gives the concrete risk: three independent
opportunities to accept a fixed invalid candidate with probability $1/4$
each yield probability $37/64$ of at least one false acceptance.

In a finite experiment, versioning the agent, improver, and evaluator
separately, fixing the improvement criterion, recording all selection calls,
and separating selection feedback from final evaluation make an RSI claim
auditable. Establishing asymptotic frontier escape remains an additional
obligation beyond demonstrating held-out performance gains.

\subsection{A sufficient contract for paired successor evaluation}\label{sec:paired}
The preceding budget becomes operational only after specifying how one call
earns its conditional guarantee. One conservative construction compares
bounded candidate and incumbent scores on fresh paired evaluation instances.
It separates sampling error from mismatch between the evaluation process and
the deployment objective. Let $G_i$ denote the expected deployment-score gain
of the proposed successor over its incumbent, for an explicitly fixed
comparison criterion. A positive observed difference need not imply $G_i>0$.

\begin{samepage}
\begin{proposition}[Fresh paired evaluation with a transfer allowance]
\label{prop:paired}
Before call $i$, let $\mathcal H_{i-1}$ fix the candidate, incumbent,
evaluation criterion, sample size $m_i\ge1$, desired gain $\tau_i\ge0$,
and allowance $b_i\ge0$. Conditional on this history, let
$Y_{i1},\ldots,Y_{im_i}\in[-1,1]$ be independent paired score differences
with common mean $\mu_i$. Assume the history-measurable deployment gain obeys
\[
G_i\ge\mu_i-b_i
\quad\text{for every reachable pre-call history.}
\]
For fixed $0<\delta_i<1$, accept only when
\begin{equation}\label{eq:paired-rule}
\overline Y_i>\tau_i+b_i+
\sqrt{\frac{2\log(1/\delta_i)}{m_i}}.
\end{equation}
Then $\Pr(\text{accept at }i\ \wedge\ G_i\le\tau_i\mid
\mathcal H_{i-1})\le\delta_i$. For at most $M$ adaptive calls, the probability
of any accepted update failing its declared gain criterion is at most
$\sum_{i=1}^{M}\delta_i$.
\end{proposition}
\end{samepage}
This is Hoeffding's inequality~\cite{Hoeffding1963} followed by
Proposition~\ref{prop:selection}; the proof is in
Appendix~\ref{app:paired-proof}. Dependence between the two scores within a
pair is allowed, as when both versions face the same task. Independence is
required between pairs conditional on the full pre-call history. The batch
size and candidate must be fixed before observing that batch. Repeatedly
checking a growing batch until it passes requires a different, sequentially
valid procedure. All candidate comparisons, including rejected proposals,
consume the declared selection budget.

If the paired evaluation directly samples the stated deployment objective,
$b_i=0$ is appropriate. Otherwise $b_i$ needs an external justification;
the sampling bound cannot establish it. A candidate-dependent environment
response or a changed judge may invalidate the transfer premise even with
fresh samples. If $b_i+\tau_i$ and sampling uncertainty exceed the available
score range, the rule abstains. It makes no claim to accept every beneficial
update, certify stage soundness, or preserve every previously accepted task.

\paragraph{Three audit cases.}
With $m=16$, $\delta=1/20$, $b=\tau=0$, and independent fair differences
$Y_j\in\{-1,1\}$, rule~\eqref{eq:paired-rule} accepts only for at least
13 positive observations. Its exact false-selection probability is
$697/65536\approx0.01064<0.05$. This checks a finite instance, not the
general proof. If instead a biased proxy always reports $Y_j=1$ while
$G=-1$, more observations reinforce a false conclusion when $b$ is
incorrectly set to zero. The valid allowance $b=2$ prevents acceptance.

A hidden reused test has a different failure mode. Draw one task index
$X$ uniformly from $\{1,\ldots,N\}$. Candidate $c_j$ passes only task $j$.
Although each fixed candidate passes with probability $1/N$, after
$j-1$ rejections the next candidate passes with conditional probability
$1/(N-j+1)$. Testing all $N$ candidates eventually passes with probability
one. Hiding $X$ does not maintain the original conditional error bound.
Appendix~\ref{app:selection-audit} gives exact enumeration and clarifies the
distinction from genuinely fresh tests. These are designed counterexamples,
not measurements of any cited agent system.

\subsection{Evolving evaluators and deployment validity}\label{sec:evolving}
Three recent studies sharpen the distinction between a task checker, an
update-selection signal, and a deployment objective.
Guo et al.~\cite{Guo2026SelfAuthored} study agents that edit both policies and
their own tests. Their external audit gate compares proposed and incumbent
policies, with a tolerance for score regression; a separate evaluation probes
deployment performance. This supports separating self-authored tests from
successor selection. It does not establish the history-conditional premise of
Proposition~\ref{prop:selection}: a hidden audit reused across proposals still
needs an adaptive-validity argument, and passing a tolerance-based gate is
not a proof of monotone deployment improvement.

Wang et al.~\cite{Wang2026Fidelity} directly study disagreements between
proxy improvement and deployment improvement, including update-induced
environment changes. Their improvement-fidelity analysis and calibrated
selection experiments address the distribution of proposed updates. This
overlaps with our motivation for a separate selection contract. Our
language-containment results concern a different object: which task inputs
an interface admits. The allowance $b_i$ in Proposition~\ref{prop:paired}
makes transfer an explicit premise; neither a structural containment proof
nor ordinary sampling concentration provides a bound on it. We do not claim
to introduce the distinction between proxy gains and deployment gains.

VeriFine~\cite{Zhou2026VeriFine} improves training-time verification through
rubric refinement and human calibration, while using a fixed reference-based
anchor for policy comparisons and independent test labels for judge
assessment. Its policy comparisons restart from a common base initialization;
the inherited state includes the curriculum and rubrics. This makes it useful
to version each component and its evaluation reference separately. A changing
judge does not by itself demonstrate frontier escape. If revised rubrics
change which outputs count as correct, a comparison also needs a common task
relation $R$, or an explicit translation between the old and new relations.
The stage-containment theorems do not compare different tasks without that
additional declaration.

\subsection{Goal information and conditional human involvement}
\label{sec:goal-information}
Task-relative soundness presupposes a specified relation $R$. A system can
improve search or verification while remaining unable to determine an
unobserved choice of goal. To isolate this issue, consider two possible
relations $R_0,R_1$ on the same observable input and output domain. The same
stage $S$ has the same exposure and acceptance rules in both situations;
neither its input nor its installed construction artifacts reveal which
relation applies. Uniform soundness across the two situations is a stronger
requirement than soundness for either declared task separately.

\begin{proposition}[Incompatible specifications force abstention]
\label{prop:goal-information}
Suppose a fixed stage $S$ satisfies task-relative soundness for both $R_0$ and
$R_1$. If an input $x$ satisfies
\[
\forall z\quad\neg\bigl(R_0(x,z)\wedge R_1(x,z)\bigr),
\]
then $S$ has no accepted exposed output on $x$, so $x$ belongs to neither
frontier. Equivalently, nonempty accepted coverage at $x$ precludes soundness
for both incompatible specifications.
\end{proposition}

\begin{proof}
Any accepted exposed record with output $z$ would imply both $R_0(x,z)$ and
$R_1(x,z)$ by the two soundness contracts, contradicting the hypothesis.
The native languages are also empty at this input because they are contained
in the corresponding frontiers.
\end{proof}

For example, let the required output be a goal bit $b$ privately selected by a
human, so $R_b(x,z)$ means $z=b$, with the observed $x$ independent of $b$.
If that person is the only permitted source of information distinguishing
the two goals, a uniformly sound agent must either obtain that information
or abstain. Providing $b$ once permits the exact constant-output rule for
the chosen goal. This establishes conditional necessity of human-origin
information, rather than a requirement for repeated live intervention.
Human ownership of the goal and exclusivity of the information channel are
explicit workflow premises. Another authorized source revealing $b$ removes
the exclusivity premise. The proposition is source-neutral and makes no
assumption that humans have a noncomputable capability.

Assistance can therefore enter during construction of the evaluator or
capability-acquisition mechanism, as well as through its subsequent calls.
The installed specification, permissions, reference data, and interpreter
belong to the declared system boundary. Their human origin can remain causally
relevant after deployment queries cease, without implying a new complexity
class or unavoidable human involvement in every improvement. Claiming that
a mechanism has not yet been formalized supplies no proof of such necessity.

This obstruction also differs from escape from a fixed closure frontier.
Changing which goal counts as correct changes the task relation and requires
an explicit comparison of tasks. If $R$ remains fixed, a proposed escape must
still establish the changed authority or exposure identified in
Theorem~\ref{thm:diagnostic}, together with its soundness and resource bounds.
New reasoning traces can improve search for accepted witnesses, as in the
counterexample-guided repair experiment, without revealing an unobserved goal
or expanding the accepted language of an unchanged stage.

\section{Verifier scope, related work, and limitations}\label{sec:scope}
\subsection{A realizable bounded-error contract}
The pointwise condition is restrictive but has concrete algorithmic instances.
Consider checking $AB=C$ for supplied $n\times n$ matrices over $\mathbb F_2$.
Draw a uniform $v\in\mathbb F_2^n$ and accept if $A(Bv)=Cv$, repeating twice
independently and requiring both tests to pass. This is the standard randomized
matrix-product check~\cite{Freivalds1979}. Correct products pass with probability
one; incorrect products pass both independent tests with probability at most
$1/4$. Thus the gap and task soundness hold uniformly. The row argument is in
Appendix~\ref{app:matrix-proof}. Each test uses quadratic field operations after the matrices are
provided. This example illustrates a valid probabilistic terminal checker;
it is not a separation from deterministic polynomial time.

By contrast, an LLM judge may confidently favor an incorrect answer on an
identifiable subset. Its probabilities can satisfy the gap while its task
soundness fails. Calibration of aggregate scores cannot prove either
universal condition. Abstention is useful only with a proved coverage and
error contract on the retained domain, or a sound fallback whose combined
acceptance is analyzed explicitly. Merely removing uncertain examples does
not remove confident mistakes. Stochastic proposal generation with an exact
terminal checker can still use the exact stage semantics if the admissible
transcript relation is justified; that semantics leaves production
probabilities unmodeled.

\subsection{Related work}
Bounded alternation and the polynomial hierarchy supply the deterministic
background~\cite{CKS1981}. Arthur--Merlin systems distinguish when verifier
coins are visible to a witness~\cite{BabaiMoran1988}; our hidden-terminal-coin,
pointwise-gap model is a particular restricted semantics. The classical
randomized hierarchy bound supplies the sandwich proof~\cite{Sipser1983,Lautemann1983}.
These results classify languages. The stage layer adds an explicit comparison
between default support and all support admitted by one interface, together
with refactoring and transition semantics.

Debate uses adversarial interaction to extend what a bounded judge can
evaluate~\cite{Irving2018}. Here a universal block means every bounded
challenge under the declared rules, and a learned debater does not automatically
realize it. Proof certificates provide another way to discharge a universal
claim on restricted families. Our certificate-substitution result isolates
the soundness, completeness, and size conditions needed for that comparison.

Test-time scaling studies how compute allocation affects performance,
including search against verifiers and iterative refinement~\cite{Snell2024}.
Reflexion uses feedback and episodic memory to improve subsequent trials
without weight updates~\cite{Shinn2023}. These mechanisms motivate distinct
search-policy and support-interface coordinates. Their empirical success
does not alone establish which stage frontier changed. Likewise, studies of
LLM judges document agreement with human preferences alongside systematic
biases~\cite{Zheng2023}, and studies of intrinsic self-correction identify
limits without external feedback~\cite{Huang2024}. Such measurements concern
particular distributions and protocols; they do not establish the pointwise
verification contracts used here. The MAST study further separates failures
of system design, inter-agent coordination, and task verification~\cite{Cemri2025}.
This motivates recording the location of a failure rather than treating
additional agents as evidence of stronger verification.

\subsection{Limitations}
The results classify declared protocol families, not latent intelligence,
training dynamics, or the probability that a model discovers a solution.
Existential witnesses and worst-case universal moves idealize possible
verification structure. A finite execution policy requires its own efficiency
and reliability analysis. There is no claim that inference scaling must
numerically plateau, that more support always helps, or that a changed stage
must escape its predecessor's frontier.

The pointwise-gap model does not establish validity for average-case judges,
distribution shift, or correlated evaluation errors. The paired selection
rule gives a sufficient condition under fresh conditional sampling and an
externally justified transfer allowance; it does not learn or validate that
allowance from the same selection data. Randomized exposure can be represented at the
level of admissible traces only when that relation is itself justified; reliable
generation probabilities require an additional semantic object. Public coins
require a different quantifier order. The strictness of the randomized hierarchy
and its relation to standard $\MA$ beyond the stated inclusion are not settled
here. The synthetic completeness constructions demonstrate expressive capacity,
and the finite audits check concrete examples. The repair experiment adds
an executable feedback policy with exhaustive finite coverage, but its exact
truth-table specifications and restricted program grammar differ from modern
LLM coding agents. No experiment here establishes learned recursive
improvement or frontier changes in large deployed systems.

\section{Amplification example and exact curves}\label{app:amplification-example}
\begin{proposition}[Existential coins can create false frontier membership]\label{prop:false}
There is a sound depth-one stage for which replacing bounded-error acceptance
by $\exists r\,\Core_S=1$ strictly enlarges the frontier and accepts inputs
outside the task domain.
\end{proposition}
A false answer can have a rare accepting tape while remaining below the bounded-error threshold. The proof is in Appendix~\ref{app:random-proofs}.
The pathology changes the acceptance semantics, not the stage. It justifies
keeping verifier randomness outside the control of witnesses and measuring
bounded error rather than treating one success as a certificate.

Figure~\ref{fig:amplification} shows how the majority rule preserves the gap
and why accepting any successful repetition has a different effect.
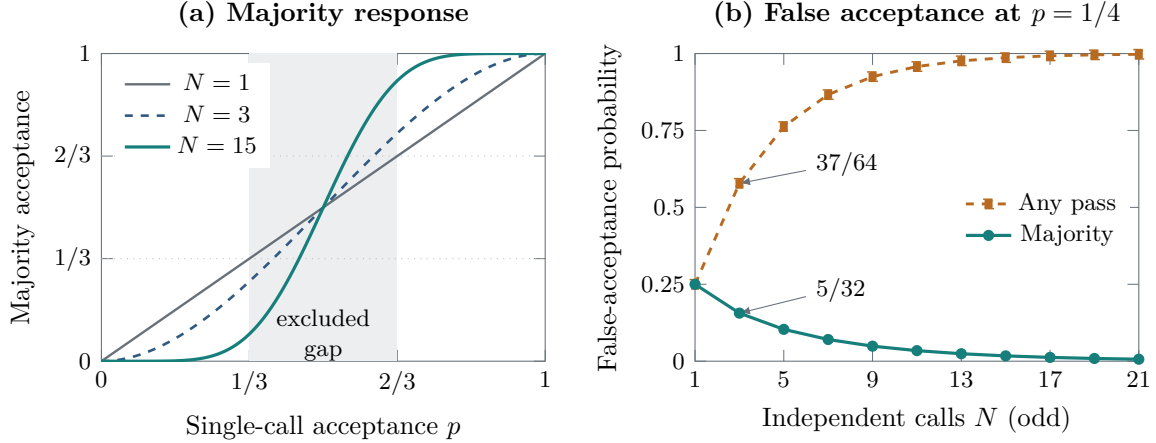
\begin{figure}[tbp]
\centering
\begin{tikzpicture}
\begin{axis}[name=amp,width=7.45cm,height=5.65cm,
  title={(a) Majority response},title style={font=\small\bfseries},
  xlabel={Single-call acceptance $p$},ylabel={Majority acceptance},
  xmin=0,xmax=1,ymin=0,ymax=1,
  xtick={0,.333333333,.666666667,1},xticklabels={$0$,$1/3$,$2/3$,$1$},
  ytick={0,.333333333,.666666667,1},yticklabels={$0$,$1/3$,$2/3$,$1$},
  tick label style={font=\footnotesize},label style={font=\small},
  axis line style={fiGray},tick style={fiGray},
  legend style={font=\footnotesize,draw=none,fill=white,at={(.03,.97)},anchor=north west},
  clip=true]
\path[fill=fiGray!12] (axis cs:.333333333,0) rectangle (axis cs:.666666667,1);
\draw[fiGray!45,dotted] (axis cs:0,.333333333)--(axis cs:1,.333333333);
\draw[fiGray!45,dotted] (axis cs:0,.666666667)--(axis cs:1,.666666667);
\addplot[fiGray,line width=.9pt] table[x=p,y=n1,row sep=\\]{
p n1 n3 n15\\
0 0 0 0\\
0.00833333333333 0.00833333333333 0.000207175925926 1.42069925749e-13\\
0.0166666666667 0.0166666666667 0.000824074074074 3.45135424561e-11\\
0.025 0.025 0.00184375 8.39090876999e-10\\
0.0333333333333 0.0333333333333 0.00325925925926 7.94782601732e-09\\
0.0416666666667 0.0416666666667 0.00506365740741 4.49051406291e-08\\
0.05 0.05 0.00725 1.82957337871e-07\\
0.0583333333333 0.0583333333333 0.00981134259259 5.94787975532e-07\\
0.0666666666667 0.0666666666667 0.0127407407407 1.63895873934e-06\\
0.075 0.075 0.01603125 3.98005141547e-06\\
0.0833333333333 0.0833333333333 0.0196759259259 8.74736542242e-06\\
0.0916666666667 0.0916666666667 0.0236678240741 1.77326456009e-05\\
0.1 0.1 0.028 3.3624887968e-05\\
0.108333333333 0.108333333333 0.0326655092593 6.02779482219e-05\\
0.116666666667 0.116666666667 0.0376574074074 0.000103004574894\\
0.125 0.125 0.04296875 0.000168888691178\\
0.133333333333 0.133333333333 0.0485925925926 0.000267106314703\\
0.141666666667 0.141666666667 0.0545219907407 0.000409244467842\\
0.15 0.15 0.06075 0.000609606807688\\
0.158333333333 0.158333333333 0.0672696759259 0.000885494493342\\
0.166666666667 0.166666666667 0.0740740740741 0.00125745099353\\
0.175 0.175 0.08115625 0.00174946009374\\
0.183333333333 0.183333333333 0.0885092592593 0.00238908725406\\
0.191666666667 0.191666666667 0.0961261574074 0.00320755565579\\
0.2 0.2 0.104 0.00423974970982\\
0.208333333333 0.208333333333 0.112123842593 0.00552414043461\\
0.216666666667 0.216666666667 0.120490740741 0.00710262889504\\
0.225 0.225 0.12909375 0.00902030577571\\
0.233333333333 0.233333333333 0.137925925926 0.0113251270924\\
0.241666666667 0.241666666667 0.146980324074 0.0140675079779\\
0.25 0.25 0.15625 0.0172998383641\\
0.258333333333 0.258333333333 0.165728009259 0.0210759261881\\
0.266666666667 0.266666666667 0.175407407407 0.0254503754241\\
0.275 0.275 0.18528125 0.0304779077712\\
0.283333333333 0.283333333333 0.195342592593 0.0362126381619\\
0.291666666667 0.291666666667 0.205584490741 0.0427073153892\\
0.3 0.3 0.216 0.0500125400538\\
0.308333333333 0.308333333333 0.226582175926 0.0581759726966\\
0.316666666667 0.316666666667 0.237324074074 0.0672415454004\\
0.325 0.325 0.24821875 0.0772486903063\\
0.333333333333 0.333333333333 0.259259259259 0.0882315984068\\
0.341666666667 0.341666666667 0.270438657407 0.100218521644\\
0.35 0.35 0.28175 0.113231130776\\
0.358333333333 0.358333333333 0.293186342593 0.127283940683\\
0.366666666667 0.366666666667 0.304740740741 0.142383813793\\
0.375 0.375 0.31640625 0.158529551124\\
0.383333333333 0.383333333333 0.328175925926 0.175711579092\\
0.391666666667 0.391666666667 0.340042824074 0.193911738761\\
0.4 0.4 0.352 0.21310318261\\
0.408333333333 0.408333333333 0.364040509259 0.233250382214\\
0.416666666667 0.416666666667 0.376157407407 0.254309248491\\
0.425 0.425 0.38834375 0.27622736443\\
0.433333333333 0.433333333333 0.400592592593 0.298944328416\\
0.441666666667 0.441666666667 0.412896990741 0.322392204575\\
0.45 0.45 0.42525 0.346496074853\\
0.458333333333 0.458333333333 0.437644675926 0.371174685979\\
0.466666666667 0.466666666667 0.450074074074 0.396341182969\\
0.475 0.475 0.46253125 0.42190391947\\
0.483333333333 0.483333333333 0.475009259259 0.447767334072\\
0.491666666667 0.491666666667 0.487501157407 0.47383288065\\
0.5 0.5 0.5 0.5\\
0.508333333333 0.508333333333 0.512498842593 0.52616711935\\
0.516666666667 0.516666666667 0.524990740741 0.552232665928\\
0.525 0.525 0.53746875 0.57809608053\\
0.533333333333 0.533333333333 0.549925925926 0.603658817031\\
0.541666666667 0.541666666667 0.562355324074 0.628825314021\\
0.55 0.55 0.57475 0.653503925147\\
0.558333333333 0.558333333333 0.587103009259 0.677607795425\\
0.566666666667 0.566666666667 0.599407407407 0.701055671584\\
0.575 0.575 0.61165625 0.72377263557\\
0.583333333333 0.583333333333 0.623842592593 0.745690751509\\
0.591666666667 0.591666666667 0.635959490741 0.766749617786\\
0.6 0.6 0.648 0.78689681739\\
0.608333333333 0.608333333333 0.659957175926 0.806088261239\\
0.616666666667 0.616666666667 0.671824074074 0.824288420908\\
0.625 0.625 0.68359375 0.841470448876\\
0.633333333333 0.633333333333 0.695259259259 0.857616186207\\
0.641666666667 0.641666666667 0.706813657407 0.872716059317\\
0.65 0.65 0.71825 0.886768869224\\
0.658333333333 0.658333333333 0.729561342593 0.899781478356\\
0.666666666667 0.666666666667 0.740740740741 0.911768401593\\
0.675 0.675 0.75178125 0.922751309694\\
0.683333333333 0.683333333333 0.762675925926 0.9327584546\\
0.691666666667 0.691666666667 0.773417824074 0.941824027303\\
0.7 0.7 0.784 0.949987459946\\
0.708333333333 0.708333333333 0.794415509259 0.957292684611\\
0.716666666667 0.716666666667 0.804657407407 0.963787361838\\
0.725 0.725 0.81471875 0.969522092229\\
0.733333333333 0.733333333333 0.824592592593 0.974549624576\\
0.741666666667 0.741666666667 0.834271990741 0.978924073812\\
0.75 0.75 0.84375 0.982700161636\\
0.758333333333 0.758333333333 0.853019675926 0.985932492022\\
0.766666666667 0.766666666667 0.862074074074 0.988674872908\\
0.775 0.775 0.87090625 0.990979694224\\
0.783333333333 0.783333333333 0.879509259259 0.992897371105\\
0.791666666667 0.791666666667 0.887876157407 0.994475859565\\
0.8 0.8 0.896 0.99576025029\\
0.808333333333 0.808333333333 0.903873842593 0.996792444344\\
0.816666666667 0.816666666667 0.911490740741 0.997610912746\\
0.825 0.825 0.91884375 0.998250539906\\
0.833333333333 0.833333333333 0.925925925926 0.998742549006\\
0.841666666667 0.841666666667 0.932730324074 0.999114505507\\
0.85 0.85 0.93925 0.999390393192\\
0.858333333333 0.858333333333 0.945478009259 0.999590755532\\
0.866666666667 0.866666666667 0.951407407407 0.999732893685\\
0.875 0.875 0.95703125 0.999831111309\\
0.883333333333 0.883333333333 0.962342592593 0.999896995425\\
0.891666666667 0.891666666667 0.967334490741 0.999939722052\\
0.9 0.9 0.972 0.999966375112\\
0.908333333333 0.908333333333 0.976332175926 0.999982267354\\
0.916666666667 0.916666666667 0.980324074074 0.999991252635\\
0.925 0.925 0.98396875 0.999996019949\\
0.933333333333 0.933333333333 0.987259259259 0.999998361041\\
0.941666666667 0.941666666667 0.990188657407 0.999999405212\\
0.95 0.95 0.99275 0.999999817043\\
0.958333333333 0.958333333333 0.994936342593 0.999999955095\\
0.966666666667 0.966666666667 0.996740740741 0.999999992052\\
0.975 0.975 0.99815625 0.999999999161\\
0.983333333333 0.983333333333 0.999175925926 0.999999999965\\
0.991666666667 0.991666666667 0.999792824074 1\\
1 1 1 1\\
};
\addlegendentry{$N=1$}
\addplot[fiNative,dashed,line width=1pt] table[x=p,y=n3,row sep=\\]{
p n1 n3 n15\\
0 0 0 0\\
0.00833333333333 0.00833333333333 0.000207175925926 1.42069925749e-13\\
0.0166666666667 0.0166666666667 0.000824074074074 3.45135424561e-11\\
0.025 0.025 0.00184375 8.39090876999e-10\\
0.0333333333333 0.0333333333333 0.00325925925926 7.94782601732e-09\\
0.0416666666667 0.0416666666667 0.00506365740741 4.49051406291e-08\\
0.05 0.05 0.00725 1.82957337871e-07\\
0.0583333333333 0.0583333333333 0.00981134259259 5.94787975532e-07\\
0.0666666666667 0.0666666666667 0.0127407407407 1.63895873934e-06\\
0.075 0.075 0.01603125 3.98005141547e-06\\
0.0833333333333 0.0833333333333 0.0196759259259 8.74736542242e-06\\
0.0916666666667 0.0916666666667 0.0236678240741 1.77326456009e-05\\
0.1 0.1 0.028 3.3624887968e-05\\
0.108333333333 0.108333333333 0.0326655092593 6.02779482219e-05\\
0.116666666667 0.116666666667 0.0376574074074 0.000103004574894\\
0.125 0.125 0.04296875 0.000168888691178\\
0.133333333333 0.133333333333 0.0485925925926 0.000267106314703\\
0.141666666667 0.141666666667 0.0545219907407 0.000409244467842\\
0.15 0.15 0.06075 0.000609606807688\\
0.158333333333 0.158333333333 0.0672696759259 0.000885494493342\\
0.166666666667 0.166666666667 0.0740740740741 0.00125745099353\\
0.175 0.175 0.08115625 0.00174946009374\\
0.183333333333 0.183333333333 0.0885092592593 0.00238908725406\\
0.191666666667 0.191666666667 0.0961261574074 0.00320755565579\\
0.2 0.2 0.104 0.00423974970982\\
0.208333333333 0.208333333333 0.112123842593 0.00552414043461\\
0.216666666667 0.216666666667 0.120490740741 0.00710262889504\\
0.225 0.225 0.12909375 0.00902030577571\\
0.233333333333 0.233333333333 0.137925925926 0.0113251270924\\
0.241666666667 0.241666666667 0.146980324074 0.0140675079779\\
0.25 0.25 0.15625 0.0172998383641\\
0.258333333333 0.258333333333 0.165728009259 0.0210759261881\\
0.266666666667 0.266666666667 0.175407407407 0.0254503754241\\
0.275 0.275 0.18528125 0.0304779077712\\
0.283333333333 0.283333333333 0.195342592593 0.0362126381619\\
0.291666666667 0.291666666667 0.205584490741 0.0427073153892\\
0.3 0.3 0.216 0.0500125400538\\
0.308333333333 0.308333333333 0.226582175926 0.0581759726966\\
0.316666666667 0.316666666667 0.237324074074 0.0672415454004\\
0.325 0.325 0.24821875 0.0772486903063\\
0.333333333333 0.333333333333 0.259259259259 0.0882315984068\\
0.341666666667 0.341666666667 0.270438657407 0.100218521644\\
0.35 0.35 0.28175 0.113231130776\\
0.358333333333 0.358333333333 0.293186342593 0.127283940683\\
0.366666666667 0.366666666667 0.304740740741 0.142383813793\\
0.375 0.375 0.31640625 0.158529551124\\
0.383333333333 0.383333333333 0.328175925926 0.175711579092\\
0.391666666667 0.391666666667 0.340042824074 0.193911738761\\
0.4 0.4 0.352 0.21310318261\\
0.408333333333 0.408333333333 0.364040509259 0.233250382214\\
0.416666666667 0.416666666667 0.376157407407 0.254309248491\\
0.425 0.425 0.38834375 0.27622736443\\
0.433333333333 0.433333333333 0.400592592593 0.298944328416\\
0.441666666667 0.441666666667 0.412896990741 0.322392204575\\
0.45 0.45 0.42525 0.346496074853\\
0.458333333333 0.458333333333 0.437644675926 0.371174685979\\
0.466666666667 0.466666666667 0.450074074074 0.396341182969\\
0.475 0.475 0.46253125 0.42190391947\\
0.483333333333 0.483333333333 0.475009259259 0.447767334072\\
0.491666666667 0.491666666667 0.487501157407 0.47383288065\\
0.5 0.5 0.5 0.5\\
0.508333333333 0.508333333333 0.512498842593 0.52616711935\\
0.516666666667 0.516666666667 0.524990740741 0.552232665928\\
0.525 0.525 0.53746875 0.57809608053\\
0.533333333333 0.533333333333 0.549925925926 0.603658817031\\
0.541666666667 0.541666666667 0.562355324074 0.628825314021\\
0.55 0.55 0.57475 0.653503925147\\
0.558333333333 0.558333333333 0.587103009259 0.677607795425\\
0.566666666667 0.566666666667 0.599407407407 0.701055671584\\
0.575 0.575 0.61165625 0.72377263557\\
0.583333333333 0.583333333333 0.623842592593 0.745690751509\\
0.591666666667 0.591666666667 0.635959490741 0.766749617786\\
0.6 0.6 0.648 0.78689681739\\
0.608333333333 0.608333333333 0.659957175926 0.806088261239\\
0.616666666667 0.616666666667 0.671824074074 0.824288420908\\
0.625 0.625 0.68359375 0.841470448876\\
0.633333333333 0.633333333333 0.695259259259 0.857616186207\\
0.641666666667 0.641666666667 0.706813657407 0.872716059317\\
0.65 0.65 0.71825 0.886768869224\\
0.658333333333 0.658333333333 0.729561342593 0.899781478356\\
0.666666666667 0.666666666667 0.740740740741 0.911768401593\\
0.675 0.675 0.75178125 0.922751309694\\
0.683333333333 0.683333333333 0.762675925926 0.9327584546\\
0.691666666667 0.691666666667 0.773417824074 0.941824027303\\
0.7 0.7 0.784 0.949987459946\\
0.708333333333 0.708333333333 0.794415509259 0.957292684611\\
0.716666666667 0.716666666667 0.804657407407 0.963787361838\\
0.725 0.725 0.81471875 0.969522092229\\
0.733333333333 0.733333333333 0.824592592593 0.974549624576\\
0.741666666667 0.741666666667 0.834271990741 0.978924073812\\
0.75 0.75 0.84375 0.982700161636\\
0.758333333333 0.758333333333 0.853019675926 0.985932492022\\
0.766666666667 0.766666666667 0.862074074074 0.988674872908\\
0.775 0.775 0.87090625 0.990979694224\\
0.783333333333 0.783333333333 0.879509259259 0.992897371105\\
0.791666666667 0.791666666667 0.887876157407 0.994475859565\\
0.8 0.8 0.896 0.99576025029\\
0.808333333333 0.808333333333 0.903873842593 0.996792444344\\
0.816666666667 0.816666666667 0.911490740741 0.997610912746\\
0.825 0.825 0.91884375 0.998250539906\\
0.833333333333 0.833333333333 0.925925925926 0.998742549006\\
0.841666666667 0.841666666667 0.932730324074 0.999114505507\\
0.85 0.85 0.93925 0.999390393192\\
0.858333333333 0.858333333333 0.945478009259 0.999590755532\\
0.866666666667 0.866666666667 0.951407407407 0.999732893685\\
0.875 0.875 0.95703125 0.999831111309\\
0.883333333333 0.883333333333 0.962342592593 0.999896995425\\
0.891666666667 0.891666666667 0.967334490741 0.999939722052\\
0.9 0.9 0.972 0.999966375112\\
0.908333333333 0.908333333333 0.976332175926 0.999982267354\\
0.916666666667 0.916666666667 0.980324074074 0.999991252635\\
0.925 0.925 0.98396875 0.999996019949\\
0.933333333333 0.933333333333 0.987259259259 0.999998361041\\
0.941666666667 0.941666666667 0.990188657407 0.999999405212\\
0.95 0.95 0.99275 0.999999817043\\
0.958333333333 0.958333333333 0.994936342593 0.999999955095\\
0.966666666667 0.966666666667 0.996740740741 0.999999992052\\
0.975 0.975 0.99815625 0.999999999161\\
0.983333333333 0.983333333333 0.999175925926 0.999999999965\\
0.991666666667 0.991666666667 0.999792824074 1\\
1 1 1 1\\
};
\addlegendentry{$N=3$}
\addplot[fiFrontier,line width=1.2pt] table[x=p,y=n15,row sep=\\]{
p n1 n3 n15\\
0 0 0 0\\
0.00833333333333 0.00833333333333 0.000207175925926 1.42069925749e-13\\
0.0166666666667 0.0166666666667 0.000824074074074 3.45135424561e-11\\
0.025 0.025 0.00184375 8.39090876999e-10\\
0.0333333333333 0.0333333333333 0.00325925925926 7.94782601732e-09\\
0.0416666666667 0.0416666666667 0.00506365740741 4.49051406291e-08\\
0.05 0.05 0.00725 1.82957337871e-07\\
0.0583333333333 0.0583333333333 0.00981134259259 5.94787975532e-07\\
0.0666666666667 0.0666666666667 0.0127407407407 1.63895873934e-06\\
0.075 0.075 0.01603125 3.98005141547e-06\\
0.0833333333333 0.0833333333333 0.0196759259259 8.74736542242e-06\\
0.0916666666667 0.0916666666667 0.0236678240741 1.77326456009e-05\\
0.1 0.1 0.028 3.3624887968e-05\\
0.108333333333 0.108333333333 0.0326655092593 6.02779482219e-05\\
0.116666666667 0.116666666667 0.0376574074074 0.000103004574894\\
0.125 0.125 0.04296875 0.000168888691178\\
0.133333333333 0.133333333333 0.0485925925926 0.000267106314703\\
0.141666666667 0.141666666667 0.0545219907407 0.000409244467842\\
0.15 0.15 0.06075 0.000609606807688\\
0.158333333333 0.158333333333 0.0672696759259 0.000885494493342\\
0.166666666667 0.166666666667 0.0740740740741 0.00125745099353\\
0.175 0.175 0.08115625 0.00174946009374\\
0.183333333333 0.183333333333 0.0885092592593 0.00238908725406\\
0.191666666667 0.191666666667 0.0961261574074 0.00320755565579\\
0.2 0.2 0.104 0.00423974970982\\
0.208333333333 0.208333333333 0.112123842593 0.00552414043461\\
0.216666666667 0.216666666667 0.120490740741 0.00710262889504\\
0.225 0.225 0.12909375 0.00902030577571\\
0.233333333333 0.233333333333 0.137925925926 0.0113251270924\\
0.241666666667 0.241666666667 0.146980324074 0.0140675079779\\
0.25 0.25 0.15625 0.0172998383641\\
0.258333333333 0.258333333333 0.165728009259 0.0210759261881\\
0.266666666667 0.266666666667 0.175407407407 0.0254503754241\\
0.275 0.275 0.18528125 0.0304779077712\\
0.283333333333 0.283333333333 0.195342592593 0.0362126381619\\
0.291666666667 0.291666666667 0.205584490741 0.0427073153892\\
0.3 0.3 0.216 0.0500125400538\\
0.308333333333 0.308333333333 0.226582175926 0.0581759726966\\
0.316666666667 0.316666666667 0.237324074074 0.0672415454004\\
0.325 0.325 0.24821875 0.0772486903063\\
0.333333333333 0.333333333333 0.259259259259 0.0882315984068\\
0.341666666667 0.341666666667 0.270438657407 0.100218521644\\
0.35 0.35 0.28175 0.113231130776\\
0.358333333333 0.358333333333 0.293186342593 0.127283940683\\
0.366666666667 0.366666666667 0.304740740741 0.142383813793\\
0.375 0.375 0.31640625 0.158529551124\\
0.383333333333 0.383333333333 0.328175925926 0.175711579092\\
0.391666666667 0.391666666667 0.340042824074 0.193911738761\\
0.4 0.4 0.352 0.21310318261\\
0.408333333333 0.408333333333 0.364040509259 0.233250382214\\
0.416666666667 0.416666666667 0.376157407407 0.254309248491\\
0.425 0.425 0.38834375 0.27622736443\\
0.433333333333 0.433333333333 0.400592592593 0.298944328416\\
0.441666666667 0.441666666667 0.412896990741 0.322392204575\\
0.45 0.45 0.42525 0.346496074853\\
0.458333333333 0.458333333333 0.437644675926 0.371174685979\\
0.466666666667 0.466666666667 0.450074074074 0.396341182969\\
0.475 0.475 0.46253125 0.42190391947\\
0.483333333333 0.483333333333 0.475009259259 0.447767334072\\
0.491666666667 0.491666666667 0.487501157407 0.47383288065\\
0.5 0.5 0.5 0.5\\
0.508333333333 0.508333333333 0.512498842593 0.52616711935\\
0.516666666667 0.516666666667 0.524990740741 0.552232665928\\
0.525 0.525 0.53746875 0.57809608053\\
0.533333333333 0.533333333333 0.549925925926 0.603658817031\\
0.541666666667 0.541666666667 0.562355324074 0.628825314021\\
0.55 0.55 0.57475 0.653503925147\\
0.558333333333 0.558333333333 0.587103009259 0.677607795425\\
0.566666666667 0.566666666667 0.599407407407 0.701055671584\\
0.575 0.575 0.61165625 0.72377263557\\
0.583333333333 0.583333333333 0.623842592593 0.745690751509\\
0.591666666667 0.591666666667 0.635959490741 0.766749617786\\
0.6 0.6 0.648 0.78689681739\\
0.608333333333 0.608333333333 0.659957175926 0.806088261239\\
0.616666666667 0.616666666667 0.671824074074 0.824288420908\\
0.625 0.625 0.68359375 0.841470448876\\
0.633333333333 0.633333333333 0.695259259259 0.857616186207\\
0.641666666667 0.641666666667 0.706813657407 0.872716059317\\
0.65 0.65 0.71825 0.886768869224\\
0.658333333333 0.658333333333 0.729561342593 0.899781478356\\
0.666666666667 0.666666666667 0.740740740741 0.911768401593\\
0.675 0.675 0.75178125 0.922751309694\\
0.683333333333 0.683333333333 0.762675925926 0.9327584546\\
0.691666666667 0.691666666667 0.773417824074 0.941824027303\\
0.7 0.7 0.784 0.949987459946\\
0.708333333333 0.708333333333 0.794415509259 0.957292684611\\
0.716666666667 0.716666666667 0.804657407407 0.963787361838\\
0.725 0.725 0.81471875 0.969522092229\\
0.733333333333 0.733333333333 0.824592592593 0.974549624576\\
0.741666666667 0.741666666667 0.834271990741 0.978924073812\\
0.75 0.75 0.84375 0.982700161636\\
0.758333333333 0.758333333333 0.853019675926 0.985932492022\\
0.766666666667 0.766666666667 0.862074074074 0.988674872908\\
0.775 0.775 0.87090625 0.990979694224\\
0.783333333333 0.783333333333 0.879509259259 0.992897371105\\
0.791666666667 0.791666666667 0.887876157407 0.994475859565\\
0.8 0.8 0.896 0.99576025029\\
0.808333333333 0.808333333333 0.903873842593 0.996792444344\\
0.816666666667 0.816666666667 0.911490740741 0.997610912746\\
0.825 0.825 0.91884375 0.998250539906\\
0.833333333333 0.833333333333 0.925925925926 0.998742549006\\
0.841666666667 0.841666666667 0.932730324074 0.999114505507\\
0.85 0.85 0.93925 0.999390393192\\
0.858333333333 0.858333333333 0.945478009259 0.999590755532\\
0.866666666667 0.866666666667 0.951407407407 0.999732893685\\
0.875 0.875 0.95703125 0.999831111309\\
0.883333333333 0.883333333333 0.962342592593 0.999896995425\\
0.891666666667 0.891666666667 0.967334490741 0.999939722052\\
0.9 0.9 0.972 0.999966375112\\
0.908333333333 0.908333333333 0.976332175926 0.999982267354\\
0.916666666667 0.916666666667 0.980324074074 0.999991252635\\
0.925 0.925 0.98396875 0.999996019949\\
0.933333333333 0.933333333333 0.987259259259 0.999998361041\\
0.941666666667 0.941666666667 0.990188657407 0.999999405212\\
0.95 0.95 0.99275 0.999999817043\\
0.958333333333 0.958333333333 0.994936342593 0.999999955095\\
0.966666666667 0.966666666667 0.996740740741 0.999999992052\\
0.975 0.975 0.99815625 0.999999999161\\
0.983333333333 0.983333333333 0.999175925926 0.999999999965\\
0.991666666667 0.991666666667 0.999792824074 1\\
1 1 1 1\\
};
\addlegendentry{$N=15$}
\node[font=\footnotesize,align=center] at (axis cs:.5,.085) {excluded\\gap};
\end{axis}
\begin{axis}[at={(7.85cm,0)},anchor=south west,width=7.45cm,height=5.65cm,
  title={(b) False acceptance at $p=1/4$},title style={font=\small\bfseries},
  xlabel={Independent calls $N$ (odd)},ylabel={False-acceptance probability},
  xmin=1,xmax=21,ymin=0,ymax=1,
  xtick={1,5,9,13,17,21},ytick={0,.25,.5,.75,1},
  tick label style={font=\footnotesize},label style={font=\small},
  axis line style={fiGray},tick style={fiGray},
  legend style={font=\footnotesize,draw=none,fill=white,at={(.98,.46)},anchor=east}]
\addplot[fiAccent,dashed,line width=1.1pt,mark=square*,mark size=1.5pt]
 table[x=n,y=any,row sep=\\]{
n majority any\\
1 0.25 0.25\\
3 0.15625 0.578125\\
5 0.103515625 0.7626953125\\
7 0.070556640625 0.866516113281\\
9 0.0489273071289 0.924915313721\\
11 0.034327507019 0.957764863968\\
13 0.0242901444435 0.976242735982\\
15 0.0172998383641 0.98663653899\\
17 0.0123847794021 0.992483053182\\
19 0.00890327930392 0.995771717415\\
21 0.006422710484 0.997621591046\\
};
\addlegendentry{Any pass}
\addplot[fiFrontier,line width=1.2pt,mark=*,mark size=1.5pt]
 table[x=n,y=majority,row sep=\\]{
n majority any\\
1 0.25 0.25\\
3 0.15625 0.578125\\
5 0.103515625 0.7626953125\\
7 0.070556640625 0.866516113281\\
9 0.0489273071289 0.924915313721\\
11 0.034327507019 0.957764863968\\
13 0.0242901444435 0.976242735982\\
15 0.0172998383641 0.98663653899\\
17 0.0123847794021 0.992483053182\\
19 0.00890327930392 0.995771717415\\
21 0.006422710484 0.997621591046\\
};
\addlegendentry{Majority}
\node[font=\footnotesize,anchor=west] at (axis cs:6,.64) {$37/64$};
\draw[-{Latex[length=1.4mm]},fiGray] (axis cs:6,.62)--(axis cs:3,.578125);
\node[font=\footnotesize,anchor=west] at (axis cs:6,.23) {$5/32$};
\draw[-{Latex[length=1.4mm]},fiGray] (axis cs:6,.21)--(axis cs:3,.15625);
\end{axis}
\end{tikzpicture}
\caption{Reliability depends on the aggregation rule. (a) The exact binomial
majority map $g_N(p)$ sends each admissible low-probability tuple to the low
side and each high-probability tuple to the high side. The shaded open
interval is excluded by the pointwise gap. Preserving this classification
preserves quantified acceptance and the two languages
(Theorem~\ref{thm:amplify}). (b) For the same fixed false tuple with
single-call probability $1/4$, independent majority reduces false acceptance,
whereas ``any pass'' gives $1-(3/4)^N$. At three calls these are $5/32$ and
$37/64$. Curves are analytic, with the tuple and exposure held fixed.}
\label{fig:amplification}
\end{figure}

\section{Repair evaluation protocol}\label{app:repair-evaluation}
\paragraph{Evaluation protocol.}
At $d=4,8$ we enumerate every program and compute every distinct admitted
target table. Within each of $M_{\rm leaf}$ and $M_{\rm all}$, both policies
are run on every such target under sixteen fixed SHA-256 orderings, using the
same ordering for each pair. No target is selected after observing performance.
These are exhaustive finite comparisons conditional on reachability in the
specified interface, not estimates of success on external coding tasks.
Blind enumeration isolates the effect of feedback; it is not a competitive
baseline for modern synthesis engines.
Coverage is separately computed by complete enumeration, including exclusion
of targets unavailable to an interface.

\section{Proof of the semantic escape diagnostic}\label{app:diagnostic-proof}
\begin{proof}[Proof of Theorem~\ref{thm:diagnostic}]
An accepted target-native $t$ that also lies in $\Exp_S^*(x)$ and satisfies
$K_S(x,t)$ would witness $x\in\Front_R(S)$, a contradiction. Under both
containment premises such a witness exists for every $x\in\Lang_R(T)$.
\end{proof}

\section{Proofs of the main results}\label{app:proofs}
The notation and assumptions are those of Section~\ref{sec:stage} and Appendix~\ref{sec:prelim}.
All object domains remain fixed unless a construction explicitly declares a
new encoding. Each proof links to its corresponding statement.

\subsection{Stage languages and deterministic recovery}\label{app:placement}

\begin{proof}[Proof of Proposition~\ref{prop:placement}]
Default support is among the permitted supports, yielding the first inclusion;
soundness yields the second. For frontier membership, guess $\sigma,t,z$ in
the leading existential block with $y_1$, check exposure, and use the predicate
$\Exp_S\wedge\Core_S$ under the remaining quantifiers. Exposure is deterministic
and independent of the later moves, so this predicate preserves the pointwise
gap and adds no alternation. The native case fixes $\sigma=\sigma_0$.
With no coins all acceptance probabilities are zero or one, and the same
construction has a deterministic polynomial-time matrix.
\end{proof}

\begin{proof}[Proof of Theorem~\ref{thm:complete}]
Choose pointwise-gap verifiers $V_0,V_1$ for the two languages, padding their
move and coin lengths to common polynomial bounds and ignoring padding. Set
$R(x,z)\iff z=1\wedge x\in L_1$. Default support exposes tag $t=0$;
one optional support exposes $t=1$. On output $z=1$ and transcript tag $t=i$,
the core runs $V_i$; it rejects other encodings. Native acceptance gives $L_0$,
and optional support gives $L_0\cup L_1=L_1$. Soundness follows from the
inclusion of languages. Deterministic $V_i$ give the exact construction.
\end{proof}

\subsection{False acceptance and majority amplification}\label{app:random-proofs}

\begin{proof}[Proof of Proposition~\ref{prop:false}]
Choose a proper language $L\in\NP$ with certificate predicate $V(x,y)$ and
set $R(x,z)\iff z=1\wedge x\in L$. Expose only $t=\eps$.
For $z=1$, let the core accept if $V(x,y)=1$ or $r=0^{n+2}$.
All other encodings are rejected. Each tuple accepts with probability either
$1$ or $2^{-(n+2)}\le1/4$, so the gap holds. Bounded-error acceptance has
frontier exactly $L$, while an existential lucky tape accepts every input.
\end{proof}

\begin{proof}[Proof of Theorem~\ref{thm:amplify}]
Let $g_N(p)$ be the probability that an odd majority accepts when each trial
accepts with probability $p$. For $N=2m+1$, direct binomial expansion gives
\[
g_{2m+3}(p)-g_{2m+1}(p)
=\binom{2m+1}{m}p^{m+1}(1-p)^{m+1}(2p-1).
\]
Consequently $g_N(p)\ge p$ for $p\ge1/2$ and $g_N(p)\le p$ for $p\le1/2$.
In particular, the high and low sides of the gap remain on their respective
sides for every odd $N$, including $N=1$. The standard independent Bernoulli
tail bound~\cite{Hoeffding1963} gives error at most
$\exp(-2N(1/6)^2)=e^{-N/18}$. Thus $D_{S^{[N]}}=D_S$ pointwise.
Applying the same quantifiers and exposure predicates proves every claimed
semantic equality and preserves task soundness.
\end{proof}

\subsection{Hierarchy bounds and derandomization}\label{app:hierarchy-proofs}

\begin{proof}[Proof of Theorem~\ref{thm:sandwich}]
The first inclusions use zero coins. For the upper bound, the total thresholded
matrix $D(x,\vec y)$ is in $\BPP$. If the last block is existential, substitute
its classical $\SP{2}$ representation $\exists u\forall v\,B$; if universal,
substitute its $\PP{2}$ representation $\forall u\exists v\,B'$.
Merging adjacent blocks of the same type leaves $k+1$ blocks. The
universal-start argument is identical with the initial quantifier reversed.
\end{proof}

\begin{proof}[Proof of Theorem~\ref{thm:regimes}]
For (i), choose $L\in\SR{1}\setminus\NP$ and use
Theorem~\ref{thm:complete} with $L_0=L_1=L$. Exact depth-one frontiers belong
to $\NP$ by Proposition~\ref{prop:placement} and cannot equal $L$.
For (ii), $D_S$ is a single uniform $\BPP$ language on the encoded tuple.
Under $\BPP=\Pclass$ it has a uniform deterministic polynomial-time decider.
Use that decider as the new terminal core, retaining all original domain
checks. Pointwise equality of matrices preserves quantified acceptance and
soundness. The same substitution proves the class equalities.
\end{proof}

\subsection{Guard-transfer invariance}\label{app:guard-proof}

\begin{proof}[Proof of Proposition~\ref{prop:guard}]
Guarding the core multiplies its acceptance probability by $g\in\{0,1\}$,
so it preserves the gap and changes the threshold matrix to $g\wedge D_S$.
The guard is independent of the moves and can pass through their nonempty
quantifier domains. Both representations therefore induce the same
\emph{fused matrix}
\[
H_S(x,\sigma,z,t,\vec y)=\Exp_S(x,\sigma,t)\wedge D_S(x,z,t,\vec y).
\]
Quantifying it with default or existential support recovers native or
frontier membership.
\end{proof}

\subsection{Depth obstructions and certificate substitution}\label{app:depth-proofs}

\begin{proof}[Proof of Theorem~\ref{thm:depth}]
For (i) choose $L\in\SP{k+1}\setminus\SP{k}$, set
$R_L(x,z)\iff z=1\wedge x\in L$, and realize native reach $L$ by the exact
case of Theorem~\ref{thm:complete}. Every sound exact depth-$k$ source has
$\Front_{R_L}(S)\in\SP{k}$ and $\Front_{R_L}(S)\subseteq L$.
If its frontier covered $L$, the two sets would be equal, contradicting
the choice of $L$.
For (ii), the randomized classes are nested by appending ignored dummy
blocks; a strict adjacent inequality therefore supplies a separator.
Use the same argument and randomized completeness and placement.
For (iii), choose $L\in\SP{k+2}\setminus\SP{k+1}$ and an exact target
with native reach $L$. Randomized source frontiers lie in $\SP{k+1}$ by
Theorem~\ref{thm:sandwich}; soundness again forces containment in $L$.
Coverage would force the forbidden equality.
\end{proof}

\begin{proof}[Proof of Proposition~\ref{prop:certificate}]
Retain the polynomial-time guard $G$, substitute the certificate equivalence
only where $G=1$, and merge $a,\pi$ into one existential block. The resulting language is in $\NP$. A $\SP{2}$-complete language in
$\NP$ puts all of $\SP{2}$ in $\NP$ under polynomial reductions; the opposite
inclusion is standard. With soundness alone only the forward implication
from certified to universal acceptance is available.
\end{proof}

\subsection{Bounded self-modification and adaptive selection}\label{app:rsi-proofs}

\begin{proof}[Proof of Proposition~\ref{prop:meta}]
Include $h,c,\sigma,t,z$ with the first existential move and use the matrix
\[
V(x,h,c,\sigma,t,z,\vec y;r)
=A(x,h,c)\wedge E(x,h,c,\sigma,t)\wedge
C(x,h,c,z,t,\vec y;r).
\]
Each conjunct is interpreted by its single uniform algorithm, including the
declared domain checks; the predicates $A,E$ precede and do not depend on
the verification moves. Invalid records reject with probability one. Valid
records inherit the candidate's gap, so $V$ is a total polynomial-time
randomized predicate with the required gap. Move-independent deterministic
guards commute with the nonempty quantifier domains. Thus quantifying $V$
recognizes exactly $L_{\rm env}$ without adding a block. For the native union,
substitute the uniformly computed default support for $\sigma$. Quantified
soundness of every admitted candidate makes the enclosing language sound for
$R$. Without coins the same construction has a deterministic matrix.
\end{proof}

\begin{proof}[Proof of Proposition~\ref{prop:selection}]
Average the conditional bounds over histories and apply the union bound.
Unused calls have empty bad events. For the last claim, the independent
bounded-variable tail inequality~\cite{Hoeffding1963} bounds each call's
majority error by $e^{-N/18}\le\delta/M$, as in
Theorem~\ref{thm:amplify}. Sum over the calls.
\end{proof}

\paragraph{Countable selection budgets.}
For a countable sequence, budgets $\delta_i=\delta/2^i$ still give total
error at most $\delta$ under the same conditional premise. This probability
statement does not imply a uniform polynomial runtime for an infinite run.

\subsection{The quota-enforced XOR family}\label{app:xor-case}
\begin{proof}[Proof of Proposition~\ref{prop:xor-case}]
For each $j$, evaluating on the basis vector $e_j$ shows that a correct
canonical program includes index $j$ if and only if $a_j=1$. Thus there is
exactly one correct canonical program, of reference count $\|a\|_0$.
Conversely that program agrees on all inputs by linearity over $\mathbb F_2$.
The core is therefore sound and complete for this restricted task.
It is the same polynomial-time algorithm at every dimension and quota;
all encodings, quota comparisons, and default supports are uniformly
polynomial-time checkable or computable. Using $O(d\log d)$ encoded
objects and an ignored move supplies an exact depth-one stage.

The unique correct program is native precisely when its length is at most
$b_0(d)$, and some allowed support admits it precisely when its length is
at most $B(d)$. This proves the two language formulas. For a target of
weight $b+1$, every source-exposed transcript has at most $b$ references,
so none can be the correct program. The expanded stage exposes that program
natively. The exclusion argument covers all source supports and transcripts;
failure to observe a successful source run is not used.

The uniform and biased policies assign positive probability to every program.
When the correct program is admissible, its masses are $2^{-d}$ and
$1/2+2^{-(d+1)}$, whose ratio is $2^{d-1}+1/2$. Taking $d$ arbitrarily large
proves the unbounded ratio without changing either language. Below the
correct program's quota both success probabilities are zero. Independence
of $M$ attempts gives $1-(1-p)^M$. These are probabilities of producing an
admissible correct output with the exact core, not of obtaining one lucky
acceptance from a noisy judge. Finally, summing $\binom dj$ over the admitted
weights gives the reported finite counts.
\end{proof}

An improver can modify either proposal policy while the harness keeps the
quota and checker fixed. Such updates preserve this envelope regardless of
whether the improver itself is editable. Admission of an arbitrary replacement
checker requires a separate soundness argument, as in
Proposition~\ref{prop:meta}. The example isolates these obligations without
claiming to explain how a learned improver discovers a useful policy.

\subsection{The randomized matrix-product check}\label{app:matrix-proof}
For the checker in Appendix~\ref{sec:scope}, correct products pass every test.
If $AB\ne C$, choose a nonzero row $d$ of $AB-C$. The dot product $dv$ is a
uniform bit for uniform $v\in\mathbb F_2^n$, so $\Pr(dv=0)=1/2$.
Acceptance implies $dv=0$, giving a one-test false-acceptance probability
at most $1/2$. Independent repetitions therefore both accept with probability
at most $1/4$. The same bound holds for every incorrect matrix triple,
establishing the pointwise gap and task soundness claimed for this checker.

\section{Representation machinery and supporting constructions}\label{app:representations}
\subsection{Simulation and support normalization}\label{app:normalization}

\begin{definition}[Simulation]\label{def:simulation}
Write $S\preceq^{\rm sim}T$ when the depths agree and a polynomial-time map
$\phi(x,t)$ maps $\Nat_S(x)$ into $\Nat_T(x)$ and $\Exp_S^*(x)$ into
$\Exp_T^*(x)$, and preserves accepted-input membership on each mapped
transcript from $S$ to $T$.
\end{definition}

\begin{proposition}[Simulation and support normalization]\label{prop:sim}
Mutual simulation preserves native reach and frontier. Every stage has a
representation with the same native reach and frontier and just two support
states, carrying optional support payloads inside tagged transcripts.
\end{proposition}
\begin{proof}
For simulation, map a witnessing transcript through $\phi$ and apply the
preserved acceptance condition; mutual simulation gives both inclusions.
For normalization, default support exposes $\langle0,t\rangle$ precisely
when $t$ was native, and optional support exposes $\langle1,\sigma,t\rangle$
precisely when $\Exp_S(x,\sigma,t)$. The new core decodes the record and runs
the old core, rejecting invalid records. A polynomial enlargement accommodates
the encoding while retaining the original decoded bounds. Decoding is
surjective onto the old exposed transcripts, generally not injective because
several support states can expose the same $t$. Existential acceptance needs
only this surjectivity, giving equality of both languages.
\end{proof}

\subsection{The local invariant of guard transfer}\label{app:local-invariant}

\begin{proposition}[Local guard invariant]\label{prop:local-invariant}
At a fixed tuple, every observation $h(e,d)$ of exposure and thresholded
acceptance that is invariant under Boolean guard transfer factors uniquely
through $e\wedge d$.
\end{proposition}

\begin{proof}
Invariance means
$h(e\wedge g,d)=h(e,g\wedge d)$. Set $e=1$ and $g=e'$ to obtain
$h(e',d)=h(1,e'\wedge d)$. Thus $h(e,d)=f(e\wedge d)$ for
$f(b)=h(1,b)$; both Boolean values are attained, giving uniqueness.
\end{proof}

The factorization is local and concerns this specified family of refactorings;
it is not a completeness theorem for every possible representation change.

\subsection{Separate same-depth escape constructions}\label{app:realizations}

\begin{proposition}[Two independent same-depth realizations]\label{prop:branches}
For any $L_0\subsetneq L_1$ in $\SR{k}$, frontier escape can be realized at
depth $k$ by changing only the core, or by changing only exposure.
\end{proposition}
\begin{proof}
Use $R(x,z)\iff z=1\wedge x\in L_1$ and verifiers $V_0,V_1$.
All constructed cores reject outputs other than $z=1$.
For a core-only change, both stages expose only $t=\eps$ and use $V_0,V_1$
respectively. For an exposure-only change, fix one core that uses $V_i$ on
tag $i$ on \emph{all} transcripts, including unexposed ones. The source exposes
only tag $0$ under every support state; the target exposes both tags already
under default support. Thus the source frontier is $L_0$, target native
reach is $L_1$, and the core is unchanged even on the newly exposed tag.
\end{proof}

The deterministic version follows with exact verifiers.

\subsection{Finite chains of same-depth frontier expansion}\label{app:chain}
\begin{proposition}
Given a finite chain $L_0\subsetneq\cdots\subsetneq L_m$ in $\SR{k}$ and
any sequence of labels $b_i\in\{\mathrm{core},\mathrm{support}\}$, there are
sound depth-$k$ stages with native and frontier languages both equal to $L_i$
such that transition $i$ changes only the component indicated by $b_i$.
\end{proposition}
\begin{proof}
Set $R(x,z)\iff z=1\wedge x\in L_m$ and choose a pointwise-gap verifier
$V_i$ for each $L_i$, with common padded bounds. Initially expose one active
tag using $V_0$. Every core rejects $z\ne1$. For every future support
transition $i$, prepare a distinct
dormant tag whose core already uses $V_{i+1}$. The core is defined on all tags,
independently of exposure. Both native and conservative exposure contain
exactly the active tags.

At a support transition, expose its prepared tag and leave the whole core
unchanged. At a core transition, keep exposure fixed and replace the verifier
on every active tag by $V_{i+1}$; leave all dormant tags unchanged. At step
$i+1$, old active languages lie in $L_{i+1}$ and either one new active tag or
every active tag realizes $L_{i+1}$. Thus the union is exactly $L_{i+1}$.
Soundness holds because even dormant languages lie within $L_m$. The strict
chain gives frontier escape at every step. Exact verifiers give the exact
version. Finiteness of the chain permits all tags and bounds to be hardwired;
no uniform infinite-chain claim follows.
\end{proof}

\section{Finite verification and reproduction}\label{app:reproduce}
\paragraph{Code availability.}
The accompanying arXiv source package places the manuscript at its root and
the executable artifacts, tests, saved results, and Lean kernel in
\texttt{anc/}. The ancillary artifact version is \texttt{2026-10-30-lean-v3}.
Its \texttt{MANIFEST.json} records SHA-256 hashes for every ancillary payload
file. Python artifacts use only the standard library and were checked with
Python 3.10.9; the Lean project specifies Lean 4.34.1 with pinned Mathlib dependencies. From \texttt{anc/},
\texttt{python3 reproduce.py} runs the regressions and compares regenerated
results with the saved records. PDF compilation separately requires a LaTeX
distribution with the packages declared in the root \texttt{main.tex}.

All vector figures and plotted coordinates are embedded in the standalone
root \texttt{main.tex}. The three analytic numeric tables are generated
by \texttt{anc/integrated/figures/generate\_data.py} using exact binomial and
rational arithmetic.
The generator checks the policy conditions against the XOR audit and records
source and data hashes; plotted curves are analytic, not sampled measurements.

The ancillary files include complementary verification artifacts. The exact artifact
enumerates the bounded patching task and checks certificate binding, hidden
counterexamples, native/frontier input sets, and accepted-output projections.
It also includes Boolean QBF and compact XOR-circuit examples. The Lean
project uses version 4.34.1 and Mathlib commit
\texttt{d13f23b723b8a846827a245b89c10fc7d3f11612}.
\texttt{QuantifiedStages.lean} checks guard transfer through input-dependent
binary prefixes, whose moves include every word up to the declared bound,
semantic containment, and support normalization.
\texttt{RandomizedStages.lean} checks exact hidden-coin probabilities and
guard transfer. \texttt{BoundedCertificates.lean} retains bounded certificate
replacement and exact Boolean-tree node counting; its resource envelope
assumes a polynomial tree-size bound. Uniform circuit generation and
Turing-machine time simulation remain separate obligations.

The selection modules use Mathlib probability and fresh-sampling kernels to
check the adaptive error budget and the paired threshold with its declared
transfer allowance, covering Propositions~\ref{prop:selection}
and~\ref{prop:paired} in the explicit probability model.
\texttt{SoundPruning.lean} checks preservation of the first correct solution
and verifier-call dominance.
\texttt{RepairTree.lean} checks the immutable-root exclusion for every tree
height. \texttt{XorQuota.lean} checks coefficient correctness, exact quota
languages, and the stated policy probabilities. \texttt{StageFamilies.lean}
checks semantic native/frontier realizations and finite sequences that
preserve exposure on core steps or preserve the core on exposure steps.
These developments do not yet discharge the uniform polynomial-time
encoding and implementation obligations of the complete paper families.
\texttt{TaskSpecification.lean} checks
Proposition~\ref{prop:goal-information} at the level of accepted records and
task-relative soundness; the proof does not require a complexity assumption
or a hypothesis about the supplier of goal information.
The implementation uses Lean 4~\cite{deMouraUllrich2021Lean}.
No proof is admitted and no custom axiom is declared. The project audits every
theorem's dependencies against propositional extensionality, classical choice,
and quotient soundness.

The source-hashed inventory \texttt{paper-coverage.json} and its
\texttt{COVERAGE.md} summary distinguish closed statements, checked components,
and open foundations. Complete formalization remains open, including the
uniform polynomial-time machine interpretation, the probabilistic-hierarchy
sandwich and conditional depth results, and general majority amplification.
Compilation is not a proof of these remaining statements. After installing
the pinned dependencies, \texttt{python3 reproduce.py --lean} from
\texttt{anc/} runs the build and axiom audit in addition to the finite checks.
The separate \texttt{--require-complete} gate fails while any paper target or
foundation is open.

From \texttt{anc/}, the exact regressions run as follows:
\begin{verbatim}
python3 -m unittest discover -s code/exact_audit -p 'test_*.py' -v
\end{verbatim}
The self-contained additional artifact is \texttt{integrated/audit.py}:
\begin{verbatim}
python3 integrated/audit.py
\end{verbatim}
It computes the four-tape noisy-checker probabilities and three-call majority
probabilities exactly, checks the exposure/internalization chain, evaluates
the quantified-soundness counterexample, exhaustively verifies Boolean guard
factorization, and tests the two-trial matrix check on all $2\times2$ Boolean
matrix triples. Its output is a JSON record with the source hash, declared
domains, and exact fractions. The finite matrix test supplements the general
row argument in Appendix~\ref{app:matrix-proof}; it does not replace that proof.

From the same \texttt{anc/} directory, the growing XOR harness and audit run as follows:
\begin{verbatim}
python3 -m unittest discover -s integrated -p 'test_xor_envelope.py' -v
python3 integrated/xor_envelope.py
\end{verbatim}
Its tests independently compare the symbolic checker to exhaustive truth
tables for $d\le4$, check malformed and mismatched records, and verify actual
transcript inclusion for internalization. Its JSON report records exact
language counts and proposal probabilities for $d\le6$, with the policy target,
selected stage, default and optional quotas, and attempt count explicitly
recorded. The escape witness is a separate target. At $d=6$, the policy target
has weight three and uses optional quota three; under default quota zero its
success probability is zero for both policies. The separate patch
report uses \texttt{exposure\_atom\_inclusion} for fact sets and
\texttt{native\_to\_frontier\_transcript\_inclusion} for the exposure premise
of Theorem~\ref{thm:diagnostic}; these diagnostics need not agree.

\paragraph{Claim boundaries.}
All finite examples use explicitly declared small domains. The polynomial
hierarchy statements concern unbounded input families and follow from their
proofs and stated assumptions. No finite audit result is reported as evidence
for a complexity-class separation, learned-verifier calibration, or a frontier
measurement of an undeclared deployment.

\section{Paired selection: proof and audit details}
\subsection{Conditional concentration}\label{app:paired-proof}
Fix a reachable pre-call history in Proposition~\ref{prop:paired}. All
quantities except the fresh paired observations are then fixed. If
$G_i\le\tau_i$, the transfer assumption implies
$\mu_i\le G_i+b_i\le\tau_i+b_i$. Acceptance under
\eqref{eq:paired-rule} therefore implies
\[
\overline Y_i-\mu_i>
\varepsilon_i:=\sqrt{2\log(1/\delta_i)/m_i}.
\]
Each independent difference has range of length two. The one-sided
Hoeffding bound gives
\[
\Pr(\overline Y_i-\mu_i>\varepsilon_i\mid\mathcal H_{i-1})
\le \exp(-m_i\varepsilon_i^2/2)=\delta_i.
\]
If $G_i>\tau_i$, the bad event is empty. Thus the conditional bound holds
at every allowed history, including those reached after rejected proposals.
Averaging and applying the union bound proves the run-level claim.
Adaptation between calls is allowed; reusing information within a call to
choose its candidate or stopping time is outside the specified construction.
The transfer premise is assumed uniformly over histories. If it is itself
only a high-probability assertion, its failure event must be accounted for
separately; the bound above must not silently omit that risk.

\subsection{Exact finite probability checks}\label{app:selection-audit}
The artifact \texttt{anc/selection/selection\_audit.py} uses rational
enumeration for the three examples in Appendix~\ref{sec:paired}. For the fair
paired null, $K\sim\mathrm{Binomial}(16,1/2)$ and
$\overline Y=2K/16-1$. The cutoff lies strictly between the values at
$K=12$ and $K=13$, so the tail is
\[
2^{-16}\sum_{k=13}^{16}\binom{16}{k}
=\frac{560+120+16+1}{65536}=\frac{697}{65536}.
\]
An independent enumeration over all $2^{16}$ observation sequences checks
the binomial calculation. The proxy example uses deterministic evaluation
differences and gain, so its discrepancy is transfer bias rather than
sampling noise; $b=0$ violates the proposition's premise.

For the reused audit example, conditional on rejecting $c_1,\ldots,c_{j-1}$,
$X$ is uniform over the remaining $N-j+1$ indices. Testing $c_j$ succeeds
with probability $1/(N-j+1)$. With an independently redrawn task at each
call, that conditional probability instead remains $1/N$, and the
probability of at least one pass in $N$ calls is $1-(1-1/N)^N$. In neither
case does this contradict a correctly applied union bound: its budget must
use an established premise, and $N$ budgets of $1/N$ already sum to one.
The example isolates the failure of the reused-test conditional premise,
not a violation of probability theory. Passing candidates in this example
have population accuracy only $1/N$, below any declared validity requirement
strictly larger than $1/N$. Tests enumerate the hidden index at $N=4,8$
and the fresh-test outcomes at $N=4$.

\section{Program-repair proofs and reproduction}\label{app:pruning}
\paragraph{Immutable-gate exclusion.}
For Proposition~\ref{prop:cut}, choose OR at the new root and AND at all
other nodes. This realizes $f_*$. For exclusion, set every left-half input
to one and every right-half input to zero. Each allowed operation preserves
zero, so induction on subtree height gives zero at the old right subtree's
root for every permitted assignment of operators. The immutable AND root
therefore returns zero. The new program returns one. The argument also
excludes the larger old mask allowing edits at every nonroot node. Constants,
negation, leaf rewiring, or a root bypass would invalidate this particular
certificate and require a different interface analysis.

\paragraph{Sound pruning.}
For Proposition~\ref{prop:prune}, a true counterexample cannot disagree with
a correct program. Every correct candidate survives each filter. Let $c_*$
be the first correct candidate in the common order. No earlier candidate can
succeed, and neither policy can skip $c_*$ as inconsistent. Both return
$c_*$; the feedback policy submits a subset of the preceding candidates.
On an unreachable task neither can obtain a valid exact acceptance. Because
exposure admits all permission-valid final records independently of the
proposal policy, neither language changes. These statements charge verifier
calls separately from the work of local candidate filtering.

\paragraph{Reproduction.}
The directory \texttt{anc/repair/} contains the two policies, exact checker,
independent tests, full saved report, and its reproduction runner.
The top-level \texttt{anc/reproduce.py} executes these checks along with
the earlier verification audits and the new selection audit. Python 3.10
or later with the standard library suffices. The repair implementation admits
only $d=4,8$ and rejects unsupported dimensions before allocating a table.
The growing-family statement is a separate proof, not an extrapolation from
these two sizes.

Programs encode operators in heap order, with children $2i+1,2i+2$ at
node $i$ and ordered input leaves. Codes $0,1,2$ mean AND, OR, XOR.
For each seed $s=0,\ldots,15$, SHA-256 of the UTF-8 string \texttt{s:}
followed by opcode bytes fixes a target-independent order. The verifier
returns the least mismatching truth-table row. Bit-parallel program semantics
are precomputed for both policies. Local consistency checks are counted
separately; precomputation and runtime are excluded from the verifier-call
metric. Counts are distinct target tables, not source strings. At the
reported sizes enumeration checks that all programs have distinct tables;
no injectivity claim is assumed at larger sizes.

Independent compiled Python execution checks the bit-parallel semantics on
all 560,304 program-input pairs. Other tests cover malformed encodings,
permissions, substituted certificate content, true counterexamples, pruning,
and unreachable targets. Every target admitted by each maximum mask is
evaluated under all sixteen orderings, so the reported means describe the
whole declared finite population. They are not confidence estimates for
external coding tasks, and no LLM runs are included.

\section{Analytic XOR family with an enforced interface}\label{sec:xor-case}
Consider synthesizing a straight-line XOR program for
$f_a(u)=\bigoplus_{j=1}^{d}a_j u_j$, given $x=(1^d,a)$ with
$a\in\mathbb F_2^d$. A candidate is a sorted list $z$ of distinct variable
indices and computes $f_z(u)=\bigoplus_{j\in z}u_j$; the empty list returns
zero. Correctness means $f_z(u)=f_a(u)$ for every $u\in\mathbb F_2^d$.
The exact core compares coefficient vectors and binds the output to the
program recorded in the transcript $t$. It rejects malformed, duplicate,
out-of-range, and mismatched records. This discharges the universal condition
in polynomial time on this restricted language.

The harness enforces a maximum number of referenced variables. Here support
selects a quota permission, rather than supplying unknown task information. It is an
integer quota $q\in\{0,\ldots,B(d)\}$, defaulting to $b_0(d)$, and
\[
\Exp_{B,b_0}(x,q,t)\iff 0\le q\le B(d),\quad
t\text{ is a canonical program},\quad |t|\le q.
\]
Here $|t|$ counts references; its binary encoding has length $O(d\log d)$.
Assume $0\le b_0(d)\le B(d)\le d$, with both quotas polynomial-time
computable. The executable harness checks exposure before invoking the same
core in every configuration. One ignored existential move gives depth one.

\begin{proposition}[Search gains and an enforced quota]\label{prop:xor-case}
For this uniform exact family,
\[
\Lang_R(S_{B,b_0})=\{(1^d,a):\|a\|_0\le b_0(d)\},\qquad
\Front_R(S_{B,b_0})=\{(1^d,a):\|a\|_0\le B(d)\}.
\]
Two proposal policies with full support on the same candidate programs can
have an unbounded ratio of per-attempt success probabilities while preserving
these languages: for every target/quota pair with $\|a\|_0\le q\le B(d)$,
the ratio is $2^{d-1}+1/2$. If $q<\|a\|_0$, both probabilities are zero. For $b=\lfloor d/2\rfloor$, changing $(B,b_0)$ from $(b,0)$ to
$(b,b)$ internalizes support. Changing it to $(b+1,b+1)$ then yields native
escape at every $d\ge1$, witnessed by any $a$ of weight $b+1$.
\end{proposition}
The proof, including exclusion of all old witnesses, is in
Appendix~\ref{app:xor-case}.

An existential accepted-input language does not record the probability of
finding its witnesses. For the policy comparison, a uniform $d$-bit mask proposes each program with
probability $2^{-d}$. A second policy uses the directly computed correct
program with probability $1/2$ and that uniform proposal otherwise. At a
quota admitting the correct program, their per-attempt success probabilities
are $2^{-d}$ and $1/2+2^{-(d+1)}$. Both retain full proposal support; disallowed
proposals are rejected by the harness. For $M$ independent attempts with
per-attempt success $p$, success is $1-(1-p)^M$. This comparison charges the
same attempt count and both policies are polynomial-time, but it is a
controlled analytic example, not a learned improvement or a wall-clock
speedup measurement. The supplied target vector already determines the correct
program, so this example tests comparison semantics rather than discovery
hardness.

The family makes two audit decisions concrete. Raising the number of attempts
changes discovery probability. Raising an enforced quota changes admissible
programs and can change the frontier even though every language here is in
$\Pclass$. A budget that is only advisory supplies no such exclusion proof.
The reference implementation enumerates all target vectors, candidate
programs, and quotas for $1\le d\le6$, tests malformed and forged records,
and checks the formulas using exact arithmetic. At $d=6$, the three
configurations above have $(|L|,|F|)$ equal to $(1,42)$, $(42,42)$, and
$(57,57)$. The general statement follows from the proof, not extrapolation
from these finite counts.

Figure~\ref{fig:xor-comparison} displays these exact finite sets alongside
the policy comparison at a fixed permitted quota. The two panels measure
different objects even though they use the same task family.
\begin{figure}[tbp]
\centering
\begin{minipage}[t]{.535\linewidth}
\vspace{0pt}\centering
\begin{tikzpicture}[x=1cm,y=1cm,font=\footnotesize]
\path[use as bounding box] (-.1,0) rectangle (8.1,5.55);
\node[font=\small\bfseries] at (4,5.23) {(a) Admitted task inputs at $d=6$};
\node at (4.5,4.64) {Target weight $\|a\|_0$};
\node at (7.35,4.15) {$(|L|,|F|)$};
\node at (2.80,4.15) {$ 0 $};
\node at (3.36,4.15) {$ 1 $};
\node at (3.92,4.15) {$ 2 $};
\node at (4.48,4.15) {$ 3 $};
\node at (5.04,4.15) {$ 4 $};
\node at (5.60,4.15) {$ 5 $};
\node at (6.16,4.15) {$ 6 $};
\node[anchor=west,align=left,text width=2.2cm,inner sep=0pt] at (0,3.55) {\textbf{Source}\\$(B,b_0)=(3,0)$};
\node at (7.35,3.55) {$(1,42)$};
\node[draw=fiNative,fill=fiNative!16,minimum width=.51cm,minimum height=.62cm,inner sep=0pt] at (2.80,3.55) {N};
\node[draw=fiFrontier,fill=fiFrontier!16,minimum width=.51cm,minimum height=.62cm,inner sep=0pt] at (3.36,3.55) {O};
\node[draw=fiFrontier,fill=fiFrontier!16,minimum width=.51cm,minimum height=.62cm,inner sep=0pt] at (3.92,3.55) {O};
\node[draw=fiFrontier,fill=fiFrontier!16,minimum width=.51cm,minimum height=.62cm,inner sep=0pt] at (4.48,3.55) {O};
\node[draw=fiGray,fill=fiGray!4,minimum width=.51cm,minimum height=.62cm,inner sep=0pt] at (5.04,3.55) {--};
\node[draw=fiGray,fill=fiGray!4,minimum width=.51cm,minimum height=.62cm,inner sep=0pt] at (5.60,3.55) {--};
\node[draw=fiGray,fill=fiGray!4,minimum width=.51cm,minimum height=.62cm,inner sep=0pt] at (6.16,3.55) {--};
\node[anchor=west,align=left,text width=2.2cm,inner sep=0pt] at (0,2.7) {\textbf{Internalized}\\$(B,b_0)=(3,3)$};
\node at (7.35,2.7) {$(42,42)$};
\node[draw=fiNative,fill=fiNative!16,minimum width=.51cm,minimum height=.62cm,inner sep=0pt] at (2.80,2.7) {N};
\node[draw=fiNative,fill=fiNative!16,minimum width=.51cm,minimum height=.62cm,inner sep=0pt] at (3.36,2.7) {N};
\node[draw=fiNative,fill=fiNative!16,minimum width=.51cm,minimum height=.62cm,inner sep=0pt] at (3.92,2.7) {N};
\node[draw=fiNative,fill=fiNative!16,minimum width=.51cm,minimum height=.62cm,inner sep=0pt] at (4.48,2.7) {N};
\node[draw=fiGray,fill=fiGray!4,minimum width=.51cm,minimum height=.62cm,inner sep=0pt] at (5.04,2.7) {--};
\node[draw=fiGray,fill=fiGray!4,minimum width=.51cm,minimum height=.62cm,inner sep=0pt] at (5.60,2.7) {--};
\node[draw=fiGray,fill=fiGray!4,minimum width=.51cm,minimum height=.62cm,inner sep=0pt] at (6.16,2.7) {--};
\node[anchor=west,align=left,text width=2.2cm,inner sep=0pt] at (0,1.85) {\textbf{Expanded}\\$(B,b_0)=(4,4)$};
\node at (7.35,1.85) {$(57,57)$};
\node[draw=fiNative,fill=fiNative!16,minimum width=.51cm,minimum height=.62cm,inner sep=0pt] at (2.80,1.85) {N};
\node[draw=fiNative,fill=fiNative!16,minimum width=.51cm,minimum height=.62cm,inner sep=0pt] at (3.36,1.85) {N};
\node[draw=fiNative,fill=fiNative!16,minimum width=.51cm,minimum height=.62cm,inner sep=0pt] at (3.92,1.85) {N};
\node[draw=fiNative,fill=fiNative!16,minimum width=.51cm,minimum height=.62cm,inner sep=0pt] at (4.48,1.85) {N};
\node[draw=fiNative,fill=fiNative!16,minimum width=.51cm,minimum height=.62cm,inner sep=0pt] at (5.04,1.85) {N};
\node[draw=fiGray,fill=fiGray!4,minimum width=.51cm,minimum height=.62cm,inner sep=0pt] at (5.60,1.85) {--};
\node[draw=fiGray,fill=fiGray!4,minimum width=.51cm,minimum height=.62cm,inner sep=0pt] at (6.16,1.85) {--};
\node[anchor=west,align=left,inner sep=0pt] at (0,1.05) {Number of\\targets};
\node at (2.80,1.05) {1};
\node at (3.36,1.05) {6};
\node at (3.92,1.05) {15};
\node at (4.48,1.05) {20};
\node at (5.04,1.05) {15};
\node at (5.60,1.05) {6};
\node at (6.16,1.05) {1};
\node[anchor=west,inner sep=0pt] at (0,.35)
  {N: native\quad O: optional only\quad --: blocked};
\end{tikzpicture}
\end{minipage}\hfill
\begin{minipage}[t]{.445\linewidth}
\vspace{0pt}\centering
\begin{tikzpicture}
\begin{axis}[width=.985\linewidth,height=5.9cm,
 title={(b) Search with the same quota},
 title style={font=\small\bfseries,align=center,text width=5.5cm},
 xlabel={Proposal attempts $M$},ylabel={At least one success},
 xmode=log,log basis x=2,xmin=1,xmax=64,ymin=0,ymax=1,
 xtick={1,2,4,8,16,32,64},xticklabels={1,2,4,8,16,32,64},
 ytick={0,.25,.5,.75,1},tick label style={font=\footnotesize},
 label style={font=\small},axis line style={fiGray},tick style={fiGray},
 legend style={font=\footnotesize,draw=none,fill=white,
   at={(.04,.47)},anchor=north west,align=left}]
\addplot[fiNative,dashed,line width=1.1pt] table[x=m,y=uniform,row sep=\\]{
m uniform mixture\\
1 0.015625 0.5078125\\
2 0.031005859375 0.757751464844\\
3 0.0461463928223 0.880768299103\\
4 0.0610503554344 0.941315647215\\
5 0.0757214436308 0.971116295113\\
6 0.090163296074 0.985783801501\\
7 0.104379494573 0.993002964801\\
8 0.11837356497 0.996556146738\\
9 0.132148978018 0.998304978473\\
10 0.145709150236 0.999165731592\\
11 0.159057444764 0.999589383518\\
12 0.172197172189 0.9997978997\\
13 0.185131591374 0.999900528759\\
14 0.197863910258 0.999951041498\\
15 0.210397286661 0.999975903238\\
16 0.222734829057 0.999988139875\\
17 0.234879597353 0.999994162595\\
18 0.246834603644 0.999997126902\\
19 0.258602812962 0.999998585897\\
20 0.270187144009 0.999999303996\\
21 0.281590469884 0.999999657436\\
22 0.292815618792 0.999999831394\\
23 0.303865374749 0.999999917014\\
24 0.314742478268 0.999999959155\\
25 0.325449627045 0.999999979897\\
26 0.335989476623 0.999999990105\\
27 0.346364641051 0.99999999513\\
28 0.356577693534 0.999999997603\\
29 0.366631167073 0.99999999882\\
30 0.376527555087 0.999999999419\\
31 0.386269312039 0.999999999714\\
32 0.395858854038 0.999999999859\\
33 0.405298559444 0.999999999931\\
34 0.414590769453 0.999999999966\\
35 0.42373778868 0.999999999983\\
36 0.432741885732 0.999999999992\\
37 0.441605293767 0.999999999996\\
38 0.450330211052 0.999999999998\\
39 0.458918801504 0.999999999999\\
40 0.467373195231 1\\
41 0.475695489055 1\\
42 0.483887747039 1\\
43 0.491952000991 1\\
44 0.499890250976 1\\
45 0.507704465805 1\\
46 0.515396583526 1\\
47 0.522968511909 1\\
48 0.53042212891 1\\
49 0.537759283146 1\\
50 0.544981794347 1\\
51 0.55209145381 1\\
52 0.559090024844 1\\
53 0.565979243206 1\\
54 0.572760817531 1\\
55 0.579436429757 1\\
56 0.586007735542 1\\
57 0.592476364674 1\\
58 0.598843921476 1\\
59 0.605111985203 1\\
60 0.611282110434 1\\
61 0.617355827459 1\\
62 0.623334642655 1\\
63 0.629220038863 1\\
64 0.635013475756 1\\
};
\addlegendentry{Uniform: $p=1/64$}
\addplot[fiFrontier,line width=1.2pt] table[x=m,y=mixture,row sep=\\]{
m uniform mixture\\
1 0.015625 0.5078125\\
2 0.031005859375 0.757751464844\\
3 0.0461463928223 0.880768299103\\
4 0.0610503554344 0.941315647215\\
5 0.0757214436308 0.971116295113\\
6 0.090163296074 0.985783801501\\
7 0.104379494573 0.993002964801\\
8 0.11837356497 0.996556146738\\
9 0.132148978018 0.998304978473\\
10 0.145709150236 0.999165731592\\
11 0.159057444764 0.999589383518\\
12 0.172197172189 0.9997978997\\
13 0.185131591374 0.999900528759\\
14 0.197863910258 0.999951041498\\
15 0.210397286661 0.999975903238\\
16 0.222734829057 0.999988139875\\
17 0.234879597353 0.999994162595\\
18 0.246834603644 0.999997126902\\
19 0.258602812962 0.999998585897\\
20 0.270187144009 0.999999303996\\
21 0.281590469884 0.999999657436\\
22 0.292815618792 0.999999831394\\
23 0.303865374749 0.999999917014\\
24 0.314742478268 0.999999959155\\
25 0.325449627045 0.999999979897\\
26 0.335989476623 0.999999990105\\
27 0.346364641051 0.99999999513\\
28 0.356577693534 0.999999997603\\
29 0.366631167073 0.99999999882\\
30 0.376527555087 0.999999999419\\
31 0.386269312039 0.999999999714\\
32 0.395858854038 0.999999999859\\
33 0.405298559444 0.999999999931\\
34 0.414590769453 0.999999999966\\
35 0.42373778868 0.999999999983\\
36 0.432741885732 0.999999999992\\
37 0.441605293767 0.999999999996\\
38 0.450330211052 0.999999999998\\
39 0.458918801504 0.999999999999\\
40 0.467373195231 1\\
41 0.475695489055 1\\
42 0.483887747039 1\\
43 0.491952000991 1\\
44 0.499890250976 1\\
45 0.507704465805 1\\
46 0.515396583526 1\\
47 0.522968511909 1\\
48 0.53042212891 1\\
49 0.537759283146 1\\
50 0.544981794347 1\\
51 0.55209145381 1\\
52 0.559090024844 1\\
53 0.565979243206 1\\
54 0.572760817531 1\\
55 0.579436429757 1\\
56 0.586007735542 1\\
57 0.592476364674 1\\
58 0.598843921476 1\\
59 0.605111985203 1\\
60 0.611282110434 1\\
61 0.617355827459 1\\
62 0.623334642655 1\\
63 0.629220038863 1\\
64 0.635013475756 1\\
};
\addlegendentry{Mixture: $p=65/128$}
\end{axis}
\end{tikzpicture}
\end{minipage}
\caption{Enforced permission and discovery probability in the XOR family.
(a) Each column represents all target vectors of one Hamming weight, with
multiplicities shown below. Making quota three the default changes native
reach from one to 42 inputs without changing the frontier. Raising the
maximum quota to four also admits the 15 weight-four targets.
(b) Both policies use the \emph{source} stage with explicit optional quota
$q=3$ and target $a=(1,1,1,0,0,0)$; under default quota zero both have success
probability zero. For $M$ independent attempts the curves are $1-(1-p)^M$. The solution mixture computes
the correct program with probability $1/2$ and samples uniformly otherwise;
it is an analytic control, not a learned RSI system. Counts and per-attempt
probabilities match the exact audit of Proposition~\ref{prop:xor-case}.}
\label{fig:xor-comparison}
\end{figure}
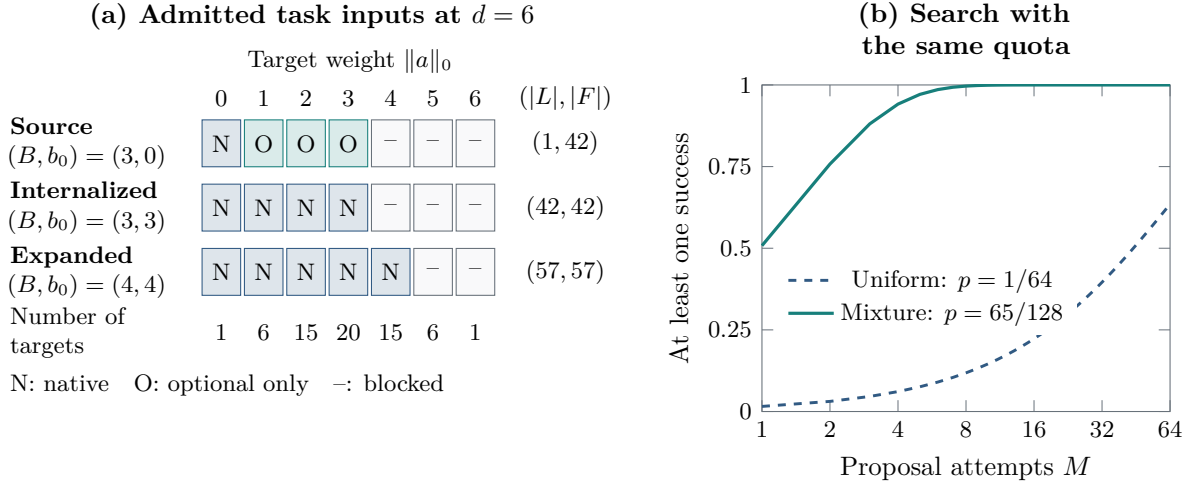

This example proves Proposition~\ref{prop:xor-case} uniformly and audits its
finite instances; the policy probabilities are analytic quantities. It does
not instantiate an adaptive learned improver. The uniform-envelope result in
Proposition~\ref{prop:meta} and the conditional selection bound in
Proposition~\ref{prop:selection} are separate theorems with separate premises.

\section{Self-improvement mechanisms and comparative evidence}\label{app:rsi-literature}
The mechanisms in Table~\ref{tab:rsi} suggest distinct stage comparisons.
The middle column summarizes reported designs; the right column gives our
proposed connection to the stage model. These connections are audit questions,
not established classifications of the cited implementations.

\begingroup\small
\begin{longtable}{@{}>{\raggedright\arraybackslash}p{.18\linewidth}
>{\raggedright\arraybackslash}p{.37\linewidth}
>{\raggedright\arraybackslash}p{.39\linewidth}@{}}
\caption{Mechanism-level comparison of self-improvement research.}
\label{tab:rsi}\\
\toprule
Research line & Reported mechanism & Stage comparison and evidence needed\\
\midrule
\endfirsthead
\multicolumn{3}{l}{\tablename~\thetable\ (continued)}\\
\toprule
Research line & Reported mechanism & Stage comparison and evidence needed\\
\midrule
\endhead
\midrule
\multicolumn{3}{r}{Continued on next page}\\
\endfoot
\bottomrule
\endlastfoot
Autonomy and inherited improvement &
Duan et al. organize RSI by control over execution, strategy, experience,
deployment adaptation, and inherited improvement mechanisms. They separate
structural from effective recursion~\cite{Duan2026LastAI}. &
An autonomy level does not determine $k$. Record accepted-language changes
separately from proposal efficiency and successor quality. A fixed enclosing
class can accommodate both inherited meta-updates and empirical gains.\\[6pt]
Prompt and context adaptation &
GEPA reflects on execution traces to evolve prompts; ACE incrementally
maintains playbooks; RSIAgent carries frozen memory into evaluation~\cite{Agrawal2026GEPA,Zhang2026ACE,Zhu2026}. &
Declare which context was already available through optional support. Compare
memory admissibility and use under matched task and output contracts.\\[6pt]
Workflow search &
AFlow uses Monte Carlo tree search over code-represented workflows and
execution feedback~\cite{Zhang2025AFlow}. &
Separate a better policy over existing traces from a changed trace grammar.
Adding a critic node does not establish universal challenge semantics.\\[6pt]
Agent and harness code &
DGM retains an archive of self-modified agents. AHE exposes editable
components and records predicted effects; Self-Harness validates proposed
edits with regression tests~\cite{Zhang2026DGM,Lin2026AHE,Zhang2026SelfHarness}. &
Version the permitted programs, tool behavior, and terminal checker. An edit
may affect exposure, acceptance, or only search efficiency; its file location
does not determine its semantic role.\\[6pt]
Editable improvers &
STOP optimizes the improver's own program. Hyperagents jointly exposes task
and meta-agent code to modification. AIDE$^2$ explores research-agent
improvement across successive versions~\cite{Zelikman2024STOP,Zhang2026Hyperagents,Srikanth2026}. &
Version $S_i$, $U_i$, and $J_i$ separately. Measure the updated improver on
common starting agents and budgets, and identify any fixed outer search or
evaluation process.\\[6pt]
Search evaluation &
SIFT aggregates pairwise judge preferences through a Bradley--Terry model
to guide parent selection and evaluation priority~\cite{Fu2026SIFT}. &
Treat ranking feedback as a search signal unless its separate acceptance
contract is established. Include all selection calls in the error and cost
accounting.\\[6pt]
Training-algorithm design &
AI4AI-Bench uses ten research repositories and reruns candidate training
algorithms from scratch under a hidden fixed evaluator~\cite{Chi2026AI4AI}. &
Keep algorithm changes distinct from runtime tuning. The benchmark's
normalized utility measures performance under its protocol; it is neither a
soundness probability nor a frontier measurement.\\[6pt]
Program discovery &
AlphaEvolve uses executable evaluation to guide program search and reports
applications across scientific domains~\cite{AlphaEvolve2026}. &
Declare the candidate language and correctness objective. Distinguish a
better discovered program from a change to the agent that discovers programs.\\
\end{longtable}
\endgroup

\subsection{Comparing task improvement and improver improvement}
The RSI survey's Headroom-Closed Index summarizes benchmark scores; its
post-2026 extension is explicitly illustrative~\cite{Duan2026LastAI}.
Neither is the closure frontier defined here. Its case studies supply
motivation and mechanisms to examine, not evidence that their implementations
satisfy our uniform verification contracts. Conversely, our closure results
do not decide whether those mechanisms produce better successors.

Hyperagents evaluates the capacity of a meta agent to generate better task
agents with an improvement@$k$ metric. In that comparison the meta agent is
held fixed, and the best child is selected using validation performance before
test performance is measured~\cite{Zhang2026Hyperagents}. This provides a concrete
precedent for separating deployed-task scores from the ability to produce
better successors. Its main experiments retain a fixed task distribution and
outer selection/evaluation processes, which delimit the reported recursion.

A comparison under the present model first fixes the task relation, candidate
grammar, evaluator, initial deployed stage, and resource accounting. It then
compares the old and new improvers with the same starting information and
reports held-out improvement as a function of evaluation calls, tokens, and
elapsed compute. A fixed number of generations alone need not equalize cost
when candidates have different runtimes. An ablation that freezes the
improver but continues task-agent search separates improvements to the search
procedure from additional search effort. These measurements concern discovery
and selection; semantic frontier escape additionally needs the evidence in
Appendix~\ref{sec:declaration}.

\subsection{Separating diagnosis, ranking, and acceptance}
AHE reports that its predicted fixes are more informative than its predicted
regressions, with many regressions remaining unforeseen~\cite{Lin2026AHE}.
SIFT explicitly treats its judge as an intermediate ranking signal rather than
a certificate of downstream superiority~\cite{Fu2026SIFT}.
Both support keeping diagnostic feedback separate from the evidence used to
accept a successor. Self-Harness makes regression testing part of that
selection loop~\cite{Zhang2026SelfHarness}.

Under our definitions, changing the selection test or evaluator must be
recorded separately from changing the candidate agent. A reused evaluation
set can become part of the effective search feedback even when the evaluator's
source code stays fixed. Fresh execution from scratch and hidden final
evaluation, as used by AI4AI-Bench~\cite{Chi2026AI4AI}, help specify an experimental
comparison; they do not by themselves establish a uniform mathematical
soundness contract. This is why Proposition~\ref{prop:selection} states a
conditional-error premise instead of inferring one from benchmark scores.

\section{Resource accounting for digital neural agents}\label{app:digital-resources}
The computational interpretation in Section~\ref{sec:computational-bridge}
uses the implemented finite-precision operations, including their rounding
and finite-bit sampling rules. It does not grant exact real arithmetic or
an unbounded computation in one neural or tool call. To see the accounting,
suppose one uniform interpreter performs at most $N(n)$ elementary operations
on at most $B(n)$ bits of state, with each operation simulable in time
$C(B(n))$. A direct digital simulation takes at most
$O(N(n)C(B(n)))$ steps, up to initialization and ordinary machine-simulation
overhead. Polynomial bounds on these quantities suffice for polynomial time.
For parallel inference, this accounting includes aggregate work and memory;
it does not identify sequential simulation time with GPU wall-clock latency.
These are sufficient modeling conditions, not measurements of a deployed
model's scaling behavior.

A fixed checkpoint may be part of a fixed program's finite description.
Its acquisition cost is then outside the deployment-time analysis. A claim
about training must count that cost separately. If model size or weights
vary with input length, the uniform results require a common effective
description with the stated resource bounds. Alternatively, an explicitly
supplied model may be part of the task input, in which case its encoding
length contributes to $n$ and defines a different task family. An unrelated
checkpoint supplied freely at each length is nonuniform advice; it cannot
silently justify a uniform class bound~\cite{BarakRunningTime}.

Tool results can be bounded support whose admissibility is checked, or tool
computation can be included in the charged implementation. The former
classifies what supplied evidence permits, without claiming that the evidence
can be efficiently obtained. An external service assumed to answer arbitrary
hard queries correctly in unit time would instead require an explicit oracle
model. Likewise, search-policy randomness affects discovery probabilities;
the hidden terminal tape in Definition~\ref{def:stage} controls the separate
verification decision. Computable neural components satisfy neither the
pointwise gap nor task soundness solely by being computable. Finally, the
polynomial bounds apply to exposure and the terminal predicate, while solving
the entire alternating acceptance condition may take much more time.

\end{document}